\documentclass{article}

\usepackage{microtype}
\usepackage{graphicx}
\usepackage{subcaption}
\usepackage{booktabs} % for professional tables

\usepackage{hyperref}

\usepackage[linesnumbered,ruled,vlined]{algorithm2e}
\usepackage{amsmath}
\usepackage{multirow}

\usepackage[preprint]{icml2026}

\usepackage{amsmath}
\usepackage{amssymb}
\usepackage{mathtools}
\usepackage{amsthm}

\usepackage{enumitem}
\usepackage[table]{xcolor}
\usepackage{subcaption}
\usepackage{tabularx}
\usepackage{array} 
\usepackage{pifont}
\usepackage{makecell}

\usepackage[capitalize,noabbrev]{cleveref}

\theoremstyle{plain}
\newtheorem{theorem}{Theorem}[section]
\newtheorem{proposition}[theorem]{Proposition}

\newtheorem{corollary}[theorem]{Corollary}
\theoremstyle{definition}

\theoremstyle{remark}

\usepackage[textsize=tiny]{todonotes}

\icmltitlerunning{Not Blind but Silenced: Rebalancing Vision and Language via Adversarial Counter-Commonsense Equilibrium}

\begin{document}

\twocolumn[
  \icmltitle{Not Blind but Silenced: Rebalancing Vision and Language \\ via Adversarial Counter-Commonsense Equilibrium}

  % It is OKAY to include author information, even for blind submissions: the
  % style file will automatically remove it for you unless you've provided
  % the [accepted] option to the icml2026 package.

  % List of affiliations: The first argument should be a (short) identifier you
  % will use later to specify author affiliations Academic affiliations
  % should list Department, University, City, Region, Country Industry
  % affiliations should list Company, City, Region, Country

  % You can specify symbols, otherwise they are numbered in order. Ideally, you
  % should not use this facility. Affiliations will be numbered in order of
  % appearance and this is the preferred way.
  \icmlsetsymbol{equal}{*}
  \icmlsetsymbol{cor}{$\dagger$}

  \begin{icmlauthorlist}
    \icmlauthor{Qingxin Xiao}{equal,A,B}
    \icmlauthor{Peilin Zhao}{equal,C}
    \icmlauthor{Yangyang Zhao}{D}
    \icmlauthor{Lingwei Dang}{A}
    \icmlauthor{Qingyao Wu}{cor,A}

    % \icmlauthor{Firstname6 Lastname6}{sch,yyy,comp}
    % \icmlauthor{Firstname7 Lastname7}{comp}
    % %\icmlauthor{}{sch}
    % \icmlauthor{Firstname8 Lastname8}{sch}
    % \icmlauthor{Firstname8 Lastname8}{yyy,comp}
    %\icmlauthor{}{sch}
    %\icmlauthor{}{sch}
  \end{icmlauthorlist}

  \icmlaffiliation{A}{South China University of Technology}
  \icmlaffiliation{B}{Institute for Super Robotics (Huangpu)}
    \icmlaffiliation{C}{Shanghai Jiao Tong University}
   \icmlaffiliation{D}{Changsha University of Science and Technology}

  \icmlcorrespondingauthor{Qingyao Wu}{qyw@scut.edu.cn}
  
  % \icmlcorrespondingauthor{Firstname2 Lastname2}{first2.last2@www.uk}

  % You may provide any keywords that you find helpful for describing your
  % paper; these are used to populate the "keywords" metadata in the PDF but
  % will not be shown in the document
  \icmlkeywords{Machine Learning, ICML}

  \vskip 0.3in
]

% this must go after the closing bracket ] following \twocolumn[ ...

% This command actually creates the footnote in the first column listing the
% affiliations and the copyright notice. The command takes one argument, which
% is text to display at the start of the footnote. The \icmlEqualContribution
% command is standard text for equal contribution. Remove it (just {}) if you
% do not need this facility.

% Use ONE of the following lines. DO NOT remove the command.
% If you have no special notice, KEEP empty braces:
\printAffiliationsAndNotice{*Equal contribution. 
    $^\dagger$ Corresponding author.}  % no special notice (required even if empty)
% Or, if applicable, use the standard equal contribution text:
% \printAffiliationsAndNotice{\icmlEqualContribution}

\begin{abstract}
 % During the decoding process of Multimodal Large Language Models (MLLMs), attention often abnormally concentrates on specific irrelevant image tokens. Existing research treats this phenomenon as invalid noise and forcibly corrects the attention distribution to compel the model to focus on key image information. However, this overlooks the fact that high-attention tokens essentially serve as critical carriers of visual information and generative narrative logic; consequently, such correction exacerbates the visual-language imbalance during the decoding phase. Addressing this issue, we adopt a "decoding-as-game" perspective to reveal that hallucinations are essentially the product of an equilibrium imbalance between language priors and visual information. We propose a training-free Adversarial Counter-Commonsense Equilibrium (ACE) framework. ACE perturbs the visual context via counter-commonsense adversarial patches. By leveraging the characteristic that authentic visual features remain stable under perturbation while hallucinated generations fluctuate with visual context changes, we construct a dynamic game decoding strategy. This strategy precisely suppresses perturbation-sensitive language priors while simultaneously compensating for stable visual information, thereby achieving visual-language balance. Extensive experiments demonstrate that ACE, as a plug-and-play strategy, introduces negligible inference overhead and effectively enhances model trustworthiness. 

 During MLLM decoding, attention often abnormally concentrates on irrelevant image tokens. While existing research dismisses this as invalid noise and forcibly redirects attention to compel focusing on key image information, we argue these tokens are critical carriers of visual and narrative logic, and such coercive corrections exacerbate visual-language imbalance. Adopting a "decoding-as-game" perspective, we reveal that hallucinations stem from an equilibrium imbalance between linguistic priors and visual information. We propose Adversarial Counter-Commonsense Equilibrium (ACE), a training-free framework that perturbs visual context via counter-commonsense patches. Leveraging the fact that authentic visual features remain stable under perturbation while hallucinations fluctuate, ACE implements a dynamic game decoding strategy. This approach precisely suppresses perturbation-sensitive priors while compensating for stable visual signals to restore balance. Extensive experiments demonstrate that ACE, as a plug-and-play strategy, enhances model trustworthiness with negligible inference overhead.
 
 % Project page: \url{https://xqxwill.github.io/ACE/}.
\end{abstract}

\begin{figure}[t]
  
  \vskip 0.2in
  \begin{center}
    % width=\textwidth 适配双栏总宽度（替代原单栏的 \columnwidth）
    % \centerline{\includegraphics[width=\textwidth]{icml_numpapers}}
    \includegraphics[width=0.9\columnwidth]{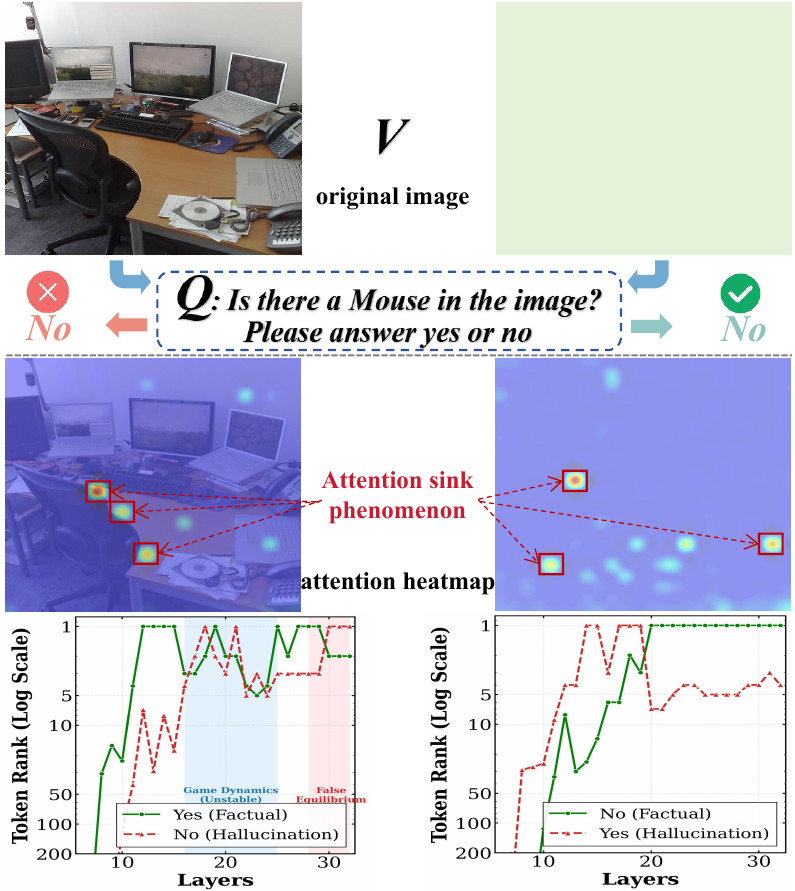} 
    \caption{We use a solid-color control to isolate the nature of Attention Sinks. Their persistence without visual content identifies them as structural artifacts independent of semantics, termed the \textbf{VSB}. Logit Lens analysis reveals a distinct contrast: the VSB remains stationary in control settings but fluctuates markedly in normal imagery. This implies visual signals are structural suppressed rather than discarded, reflecting a competition between visual and linguistic priors. \textbf{Appendix~\ref{app:logit_dynamics}} provides detailed empirical validation of this non-cooperative game mechanism.}

    \label{fig1}
  \end{center}
  \vskip -0.2in 
\end{figure}

\section{Introduction}

In recent years, Multimodal Large Language Models (MLLMs) have demonstrated remarkable advancements in comprehending complex visual scenarios and generating coherent, contextually relevant descriptions for visual inputs \cite{lu2024deepseek, lai2024lisa, geng2024instructdiffusion}. Despite the robust versatility endowed by their superior performance across diverse visual tasks, these models are universally plagued by the challenge of "hallucination" \cite{bai2024hallucination, guan2024hallusionbench}. Specifically, this refers to instances where the generated text, while semantically fluent and coherent, contains erroneous content that contradicts visual facts. Common manifestations include identifying non-existent objects, omitting existing targets, or misrepresenting attributes such as position, color, and quantity \cite{liu2024survey}. This phenomenon not only severely undermines the trustworthiness of MLLMs but also significantly impedes their broader deployment in real-world scenarios \cite{chen2024towards}.

In investigating the origins of MLLM hallucinations, existing research predominantly attributes them to abnormal attention collapse (Attention Sinks) during the decoding phase \cite{park2025second, woo2025don}(as shown in fig. ~\ref{fig1}). These studies observe that the model's focus excessively concentrates on irrelevant tokens, such as image backgrounds \cite{zhao2024mitigating, li2025hidden, zoulook, pmlr-v267-luo25b, che2025hallucinatory, li2025mitigating}, and attempt to alter the attention distribution through methods like forced masking to compel the model to attend to specific image regions \cite{park2025second, li2025mitigating, cho2025revisitseediscloselanguage}. However, this perspective confounds the causality of hallucination generation, as anomalies in attention distribution are often merely manifestations of decoding biases rather than their root cause \cite{min2025mitigating, ICLR2025_109cf25c}. We propose a core hypothesis: the essence of hallucination is a structural imbalance resulting from an asymmetric game between parametric language priors and non-parametric visual information during the decoding process. Our research reveals that the visual flow does not decay with layer depth but persists throughout. However, in the effort to maintain generation coherence, the model tends to sacrifice high-entropy visual information by implicitly suppressing and temporarily storing it within the Sink regions, thereby collapsing onto low-entropy semantic prior paths. Consequently, the model is effectively rendered "silenced" rather than "blind".

Consequently, Attention Sinks should not be regarded merely as triggers for hallucination or meaningless noise. On the contrary, we propose that they are \textbf{Visual Semantic Buffers (VSB)} inevitably derived by the autoregressive decoding mechanism to maintain contextual narrative logic. Essentially, VSBs serve as information buffers established by the model when processing high-dimensional visual signal flows. When an imbalance occurs between visual information and language priors, the model does not discard authentic visual features; instead, it temporarily stores them within these high-attention VSBs to preserve implicit contextual narrative coherence. This implies that VSBs contain a substantial amount of non-explicitly decoded authentic visual information and contextual semantic information. Crucially, our experiments demonstrate that the emergence of attention anchors is an inherent requirement of autoregressive generation; these VSBs persist throughout the decoding process and exhibit significant instability. Blindly suppressing or masking current VSBs fails to eliminate this mechanism. Instead, it compels the model to dynamically migrate attention to new, unpredictable locations, triggering the stochastic emergence of secondary VSBs. Therefore, in the absence of effective information migration or decoupling mechanisms, blindly suppressing or re-weighting VSBs effectively disrupts the model's visual caching and attention homeostasis. This not only fails to decouple hallucinations but exacerbates them.

To address this, we propose \textbf{Adversarial Counter-Commonsense Equilibrium (ACE)}. We define the standard inference as the \textbf{Original Input Stream (OIS)}. By introducing counter-commonsense adversarial patches into non-subject regions, we construct a \textbf{Counter-commonsense Interference Stream (CIS)}. Leveraging the insight that authentic features remain robust under interference while hallucinations drift, we filter unstable priors to construct a \textbf{Decoupled Visual Stream (DVS)}. We inject the DVS into the active mid-layer hidden states to calibrate the trajectory and implement a dynamic game decoding at the final layer. This strategy precisely penalizes entrenched priors and rewards stable visual features, achieving a dynamic visual-language equilibrium.

Notably, ACE is a training-free, plug-and-play strategy requiring no external networks. Our contributions are threefold:
\begin{itemize}[noitemsep, topsep=0pt, parsep=0pt, leftmargin=*]
    \item We unveil that hallucination results from an imbalance in the dynamic game between language priors and visual information within VSBs.
    \item We propose the ACE framework, employing counter-commonsense perturbations and DVS to restore equilibrium through dynamic game decoding.
    \item We validate ACE's effectiveness in enhancing trustworthiness across complex scenarios on mainstream benchmarks, including POPE \cite{li-etal-2023-evaluating}, CHAIR \cite{rohrbach-etal-2018-object}, and MME \cite{fu2025mme}.
\end{itemize}

% --- Main Text: Related Works ---

\begin{figure*}[ht]
  \vskip 0.2in
  \begin{center}
    % width=\textwidth 适配双栏总宽度（替代原单栏的 \columnwidth）
    % \centerline{\includegraphics[width=\textwidth]{icml_numpapers}}
    \includegraphics[width=1\textwidth]{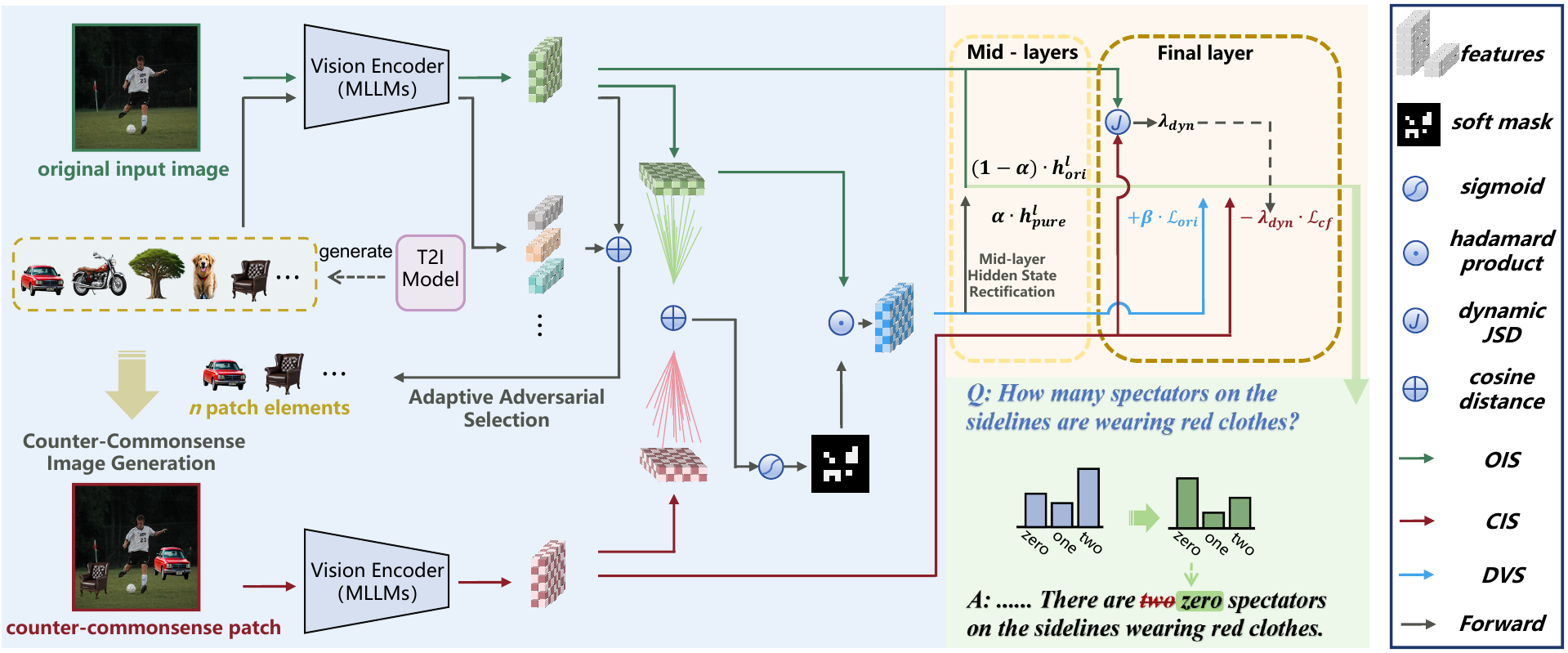} 
    \caption{\textbf{Overview of the ACE framework.} 
Starting with the \textbf{Adaptive Adversarial Selection}, ACE retrieves and \textbf{pastes counter-commonsense patches} (e.g., pasting a ``car'' into a ``soccer field'') to construct the CIS. 
By contrasting the CIS with the OIS via cosine similarity in the deep feature space, ACE employs a soft gating mechanism to decouple environment-invariant features, forming the DVS. 
Finally, by implementing hierarchical synergistic interventions at the mid-layers (hidden state rectification) and the final layer (dynamic game), where ACE dynamically penalizes perturbation-insensitive hallucinations while explicitly rewarding visually faithful features, the framework effectively restores the visual-language equilibrium. 
See Sec. 4 for more details.}

    \label{fig2}
  \end{center}
  \vskip -0.2in % 可选：微调图片与下文的间距
\end{figure*}

\section{Related Works}
\label{sec:related_works}

\textbf{Inference-time Decoding Strategies.} 
Inference-time decoding strategies operate without parameter updates, suppressing hallucinations by solely intervening in logits or attention flows, thereby garnering significant attention for their efficiency and flexibility. A dominant paradigm is Contrastive Decoding, which calibrates model outputs by introducing a ``negative'' distribution. Methods such as VCD~\cite{leng2024mitigating}, ICD~\cite{kim2024instructive, wang-etal-2024-mitigating}, and OPERA~\cite{huang2024opera} penalize tokens that receive high scores in hallucination-prone contexts; meanwhile, DOLA~\cite{ICLR2024_edc36117} and TCD~\cite{back2025watermarking} leverage layer-wise discrepancies to eliminate linguistic priors. 
To investigate the intrinsic roots of hallucination, research has increasingly delved into the internal microscopic attention mechanisms of MLLMs. Microscopic attention correction focuses on the ``Attention Sink'' phenomenon, where attention collapses onto irrelevant regions such as backgrounds. This phenomenon is identified as the root cause of the model's failure to perceive the image accurately and its subsequent degeneration into prior-based completion~\cite{zhao2024mitigating, li2025hidden, zoulook, pmlr-v267-luo25b, che2025hallucinatory}. Approaches like IMCCD~\cite{li2025mitigating}, ReVisiT~\cite{cho2025revisitseediscloselanguage}, SECOND~\cite{park2025second}, Pai~\cite{liu2024paying}, AVISC~\cite{woo2025don}, and SID~\cite{ICLR2025_3cc87f2b} employ masking or re-weighting strategies to forcibly redirect the model's attentional focus.

\textbf{Gap and Our Approach.} 
However, existing methods predominantly regard Attention Sinks as adversarial noise to be suppressed \cite{jiang2025visiontransformersdontneed, ICLR2024_0b408293}, overlooking their structural role as \textbf{Visual Semantic Buffers (VSBs)} holding implicit evidence. To prevent information loss, we propose ACE, which leverages feature stability under counter-commonsense perturbations to \textit{decouple} rather than \textit{suppress} these buffers, thereby achieving robust hallucination mitigation while preserving attention homeostasis. (See Appendix~\ref{app:background} for comprehensive background details).

\section{Preliminaries: Problem Formulation}
\label{sec:preliminaries}

In this section, we reframe the decoding process of MLLMs. We move beyond the standard autoregressive definition to establish a physics-grounded "Game Theory" perspective. Here, ``physics-grounded'' specifically refers to the dominance of magnitude-driven dynamics (e.g., norm inflation and saturation effects) within the Transformer attention mechanism, rather than classical physical systems. This formulation reveals the underlying mechanism of hallucinations and sets the theoretical stage for our proposed intervention. To provide empirical motivation for this perspective, we conducted a preliminary exploratory analysis, as visualized in Figure~\ref{fig1}. By tracking the logit ranking evolution of factual versus hallucinated tokens, we observe a distinct pattern disparity: the solid-color control maintains a ``stationary stability,'' whereas the normal imagery exhibits significant ``oscillations'' in intermediate layers. This observation suggests that the decoding process may involve a dynamic competition between visual information and linguistic priors, rather than a simple linear decay. This serves as an intuitive basis for the formal game-theoretic modeling presented below. (Rigorous mathematical derivations detailed in Appendix~\ref{app:game_theory}).

% \subsection{The Physical Basis: Visual Semantic Buffer (VSB)}
% Standard hypotheses often assume that visual information undergoes "monotonically linear decay" in deep Transformer layers. However, our theoretical analysis (proofs provided in Appendix~\ref{app:vsb_theory}) refutes this view. We demonstrate that driven by the norm inflation of register neurons, the context vector $\mathbf{c}_t$ does not lose visual information but compresses it into a high-norm structure, which we define as the VSB:
% \begin{equation}
%     \mathbf{c}_t \approx \mathcal{V}_{global} + \mathbf{n}_{bias}.
% \end{equation}
% Here, $\mathcal{V}_{global}$ represents the authentic visual semantics effectively compressed within the buffer. In contrast, $\mathbf{n}_{bias}$ denotes the high-norm Linguistic and Contextual Inertia. This formulation proves that hallucinations arise not because the model is "blind," but because the authentic signal $\mathcal{V}_{global}$ is structurally silenced by the dominant magnitude of $\mathbf{n}_{bias}$. To intuitively validate this, we analyze the topological reorganization of attention heatmaps in Appendix~\ref{app:visualization}, confirming that the VSB is a physically responsive carrier rather than a static sink. While the visual signal $\mathcal{V}_{global}$ is latent, it is structurally suppressed by the cumulative bias $n_{bias}$ across layers. Our mid-layer rectification serves as a representational catalyst, preventing the irreversible collapse of visual evidence into linguistic priors.

\subsection{The Physical Basis: Visual Semantic Buffer (VSB)}
Contrary to the assumption of "monotonically linear decay" in deep Transformers, our analysis (Appendix~\ref{app:vsb_theory}) suggests that visual information is not lost but compressed into a high-norm structure—the VSB—driven by the norm inflation of register neurons:
\begin{equation}
    \mathbf{c}_t \approx \mathcal{V}_{global} + \mathbf{n}_{bias}.
\end{equation}
Here, $\mathcal{V}_{global}$ captures the latent visual semantics, while $\mathbf{n}_{bias}$ represents the dominant linguistic and contextual inertia. This formulation reveals that hallucinations stem from the structural silencing of $\mathcal{V}_{global}$ by the high-magnitude $\mathbf{n}_{bias}$, rather than a lack of visual perception. Attention heatmap visualizations (Appendix~\ref{app:visualization}) further confirm the VSB as a responsive carrier rather than a static sink. Since $\mathcal{V}_{global}$ is progressively suppressed by cumulative bias, our mid-layer rectification acts as a representational catalyst, preventing the irreversible collapse of visual information into linguistic priors.

\subsection{Decoding as a Non-cooperative Game}
Building upon the VSB structure, the decoding process can be physically modeled as a dynamic non-cooperative game between two virtual agents competing for the generation right of the next token $y_t$:
\begin{itemize}
    \item \textbf{Agent V (The Visual Agent):} Strives to maximize the likelihood of tokens grounded in visual facts $P(y_t \mid \mathcal{V}_{global})$.
    \item \textbf{Agent L (The Language Agent):} Strives to maximize the likelihood of tokens driven by narrative priors $P(y_t \mid \mathbf{n}_{bias})$.
\end{itemize}

\paragraph{Motivation for ACE.}
In standard decoding, the accumulated inertia of Agent L often overwhelms Agent V, forcing the system into a False Equilibrium where $P(y_{hallucination}) > P(y_{fact})$. To resolve this structural imbalance, we cannot rely on the model's self-correction. Instead, we must introduce an external intervention mechanism to reshape the payoff matrix of this game. This necessitates the proposal of our ACE framework.

\begin{table*}[t]
\centering
\caption{\textbf{Evaluation results on the POPE benchmark across four MLLMs.} We report the Accuracy and F1 scores computed across adversarial, popular, and random splits. The best results are highlighted in \textbf{bold}.}
\label{tab1}
\begin{tabular}{l c p{1mm} c c p{0.5mm} c c p{1mm} c c p{1mm} c c}
\hline
Setting & Method & & \multicolumn{2}{c}{Shikra} & & \multicolumn{2}{c}{LLaVA-1.5} & & \multicolumn{2}{c}{InstructBLIP} & & \multicolumn{2}{c}{LLaVA-NeXT} \\
\cline{4-5} \cline{7-8} \cline{10-11} \cline{13-14}
      & POPE & & Acc$\uparrow$ & F1$\uparrow$ & & Acc$\uparrow$ & F1$\uparrow$ & & Acc$\uparrow$ & F1$\uparrow$ & & Acc$\uparrow$ & F1$\uparrow$ \\
\hline

% ===================== Group 1 =====================
\multirow{5}{*}{Adversarial}
 & Greedy      & & 78.25 & 79.42 & & 80.13 & 80.96 & & 75.03 & 76.84 & & 80.12 & 80.53 \\
 & VCD         & & 78.97 & 79.69 & & 80.15 & 81.04 & & 76.87 & 78.36 & & 81.09 & 81.03 \\
 & OPERA       & & 78.85 & 79.72 & & 80.98 & 81.35 & & 75.12 & 76.71 & & 81.94 & 81.89 \\
 & SID         & & 78.83 & 79.36 & & 83.05 & 82.32 & & \textbf{78.65} & 79.63 & & 84.08 & 82.92 \\
\rowcolor[HTML]{C0C0C0}
\cellcolor{white} & \textbf{ACE(ours)} & & \textbf{79.54} & \textbf{79.98} & & \textbf{84.65} & \textbf{82.59} & & 78.39 & \textbf{79.64} & & \textbf{85.68} & \textbf{82.89} \\ 
\hline

% ===================== Group 2 =====================
\multirow{5}{*}{Popular}
 & Greedy      & & 82.29 & 82.97 & & 83.42 & 84.17 & & 80.14 & 75.66 & & 85.48 & 86.17 \\
 & VCD         & & 81.48 & 82.34 & & 85.73 & \textbf{85.77} & & 79.32 & 79.64 & & 83.11 & 84.23 \\
 & OPERA       & & 82.60 & \textbf{83.19} & & 84.65 & 84.60 & & 79.36 & \textbf{80.45} & & 83.96 & 84.88 \\
 & SID         & & 82.47 & 81.65 & & 84.73 & 85.14 & & 80.07 & 77.40 & & 86.32 & 86.09 \\
\rowcolor[HTML]{C0C0C0}
\cellcolor{white} & \textbf{ACE(ours)} & & \textbf{83.29} & 83.01 & & \textbf{85.96} & 85.38 & &\textbf{81.84} & 79.61 & & \textbf{88.83} & \textbf{88.32} \\
\hline

% ===================== Group 3 =====================
\multirow{5}{*}{Random}
 & Greedy      & & 84.03 & 84.24 & &89.62 &89.14 & &84.94 &84.33 & &89.74 &89.03 \\
 & VCD         & & 84.72 & 84.60 & &89.55 &\textbf{89.31} & &84.82 & 83.79 & &88.60 &87.89 \\
 & OPERA       & & 84.25 & 84.68 & &88.64 &88.43 & &84.86 &84.12 & &90.17 &89.74 \\
 & SID         & & 84.50 & 84.74 & &89.09 &88.83 & &86.34 &86.83 & &89.96 &90.10 \\
\rowcolor[HTML]{C0C0C0}
\cellcolor{white} & \textbf{ACE(ours)} & &\textbf{84.93} &\textbf{85.21} & &\textbf{90.87} &88.93 & &\textbf{86.96} &\textbf{86.98} & &\textbf{91.28} &\textbf{90.63} \\
\hline

\end{tabular}
\vskip -0.1in
\end{table*}

\section{Method}
\label{sec:method}

Drawing upon the game-theoretic problem formulation in Section~\ref{sec:preliminaries}, we propose the ACE framework. As illustrated in Fig.~\ref{fig2}, ACE departs from the traditional single-stream paradigm by constructing a synergistic three-stream architecture: the \textbf{Original Input Stream (OIS)}, the \textbf{Counter-commonsense Interference Stream (CIS)}, and the \textbf{Decoupled Visual Stream (DVS)}.

Through the dynamic interaction of these streams, ACE deeply decouples the entangled representations within the VSB, separating visual facts from linguistic priors to restore a factual equilibrium.

\subsection{Negative Probe: CIS (Penalizing Agent L)}
To identify and penalize hallucinations driven by Narrative Inertia (Agent L), ACE constructs the CIS. The core logic relies on the \textbf{Sensitivity Disparity} between the two agents. Tokens driven purely by linguistic priors exhibit "irrational insensitivity" to drastic visual perturbations, whereas visual-grounded tokens are highly sensitive to such changes.

\paragraph{Offline Adversarial Library Construction.}
To ensure perturbations induce maximal semantic conflict, we pre-construct an Adversarial Object Library $\mathcal{B}$ using generative models. This library contains $K$ diverse objects (such as \textit{polar bears, spaceships, volcanoes}) spanning mutually exclusive categories. We pre-cache their feature vectors to enable efficient retrieval during inference. (See Appendix~\ref{app:library_construction} for details).

\paragraph{Adaptive Adversarial Selection.}
For an input image $I_{raw}$, we employ a "farthest neighbor selection" strategy within the shared feature manifold to identify a disturbance source $p^* \in \mathcal{B}$ that is semantically orthogonal to the global image feature $\mathbf{v}_{raw}$:
\begin{equation}
    p^* = \arg\min_{p_k \in \mathcal{B}} \text{CosSim}(\mathbf{v}_{raw}, \mathbf{v}_k).
\end{equation}
This ensures that the adversarial probe resides at the semantic antipode of the current context, maximizing the potential representational displacement. For example, the system may select a "fire truck" to perturb an "underwater" scene. Such semantic orthogonality guarantees the potency of the interference: a low JSD reliably indicates a lack of visual grounding, as even a maximally distant stimulus fails to displace the stubborn linguistic prior, thereby exposing the hallucination.

\paragraph{Counter-Commonsense Image Generation.}
We generate the Counter-Commonsense image $I_{cf}$ by pasting $p^*$ onto the non-salient background of $I_{raw}$. To protect the main subject, we utilize a morphological denoising strategy to extract a background mask $\mathcal{M}_{noise}$:
\begin{equation}
    I_{cf} = I_{raw} \cdot (1 - \mathcal{M}_{noise}) + \sum_{j=1}^{N} (p^* \cdot M_{patch}^{(j)}).
\end{equation}
We constrain $M_{patch}^{(j)}$ to be binary and non-overlapping within $\mathcal{M}_{noise}$ to preserve local pixel intensity. The resulting feature stream serves as a "Negative Probe" for causal inference: if the model's output distribution remains indifferent to this drastic and semantically discordant visual change, it provides strong evidence that the generation is a hallucination blindly dominated by Agent L.

\subsection{Positive Anchor: DVS (Rewarding Agent V)}
Complementary to the penalization mechanism of CIS, ACE constructs the DVS to ``reward loyalty.'' This module is founded on the Semantic Rigidity Hypothesis. In the deep feature space, authentic visual objects (Agent V) maintain their feature direction under background perturbations, whereas context noise drifts significantly.

We input $I_{raw}$ and $I_{cf}$ into the visual encoder to obtain feature sequences $\mathbf{F}_{raw}$ and $\mathbf{F}_{cf}$. Note that the token-wise comparison is geometrically valid, as both inputs share identical positional embeddings and architectural constraints. We quantify the semantic rigidity of each token $i$ via point-wise cosine similarity in the normalized space:
\begin{equation}
    S^{(i)} = \frac{\mathbf{F}_{raw}^{(i)} \cdot \mathbf{F}_{cf}^{(i)}}{\|\mathbf{F}_{raw}^{(i)}\|_2 \|\mathbf{F}_{cf}^{(i)}\|_2}.
\end{equation}
To translate this similarity into a differentiable decoupling mask $\mathbf{M}$, we apply a soft gating mechanism controlled by a temperature coefficient $\kappa$:
\begin{equation}
    \mathbf{M}^{(i)} = \sigma \left( \kappa \cdot (S^{(i)} - \tau) \right).
\end{equation}
Finally, we perform physical isolation via the Hadamard product to obtain the high-purity visual stream: $\mathbf{V}_{iso} = \mathbf{F}_{raw} \odot \mathbf{M}$
% \begin{equation}
%     \mathbf{V}_{iso} = \mathbf{F}_{raw} \odot \mathbf{M}
% \end{equation}
$\mathbf{V}_{iso}$ strips away environmental dependencies and retains only robust visual facts. This serves as a reliable ``Positive Anchor'' for Agent V in the subsequent game.

\subsection{Hierarchical Synergistic Intervention}
ACE enforces the decoupling mechanism throughout the decoding lifecycle by intervening in both the hidden state space (mid-layers) and the probability space (final layer).

\paragraph{Mid-layer Hidden State Rectification.}
To prevent Agent L from suppressing visual signals early in the network, we inject the purified visual state $\mathbf{h}_{iso}^{(l)}$ (derived from $\mathbf{V}_{iso}$) into the original stream $\mathbf{h}_{raw}^{(l)}$ at intermediate layers. Using an injection ratio $\alpha$, the update rule is:
\begin{equation}
    \mathbf{h}_{raw}^{(l)} \leftarrow (1 - \alpha) \cdot \mathbf{h}_{raw}^{(l)} + \alpha \cdot \mathbf{h}_{iso}^{(l)}.
\end{equation}
This operation physically reinforces the weight of $\mathcal{V}_{global}$ within the VSB, calibrating the inference trajectory before the final decision phase.

\paragraph{Final-layer Dynamic Game Decoding.}
At the critical logits generation phase, ACE constructs the dynamic game field. Based on the equilibrium theorem derived in Appendix Theorem~\ref{thm:ace_equilibrium}, the final decoding logits $\mathcal{L}_{final}$ are formulated as a modified total utility function:
\begin{equation}
    \mathcal{L}_{final} = \mathcal{L}_{raw} - \lambda_{dyn}(t) \cdot \mathcal{L}_{cf} + \beta \cdot \mathcal{L}_{iso}.
\end{equation}
This formula applies two critical game-theoretic corrections to the standard generation process:

\begin{enumerate}[noitemsep, topsep=0pt, parsep=0pt, leftmargin=*]
    \item \textbf{Dynamic Penalty Term ($\lambda_{dyn}$):} We use the Jensen-Shannon Divergence (JSD) between the distributions of OIS ($P_{raw}$) and CIS ($P_{cf}$) to quantify the model's sensitivity. The adaptive penalty coefficient is defined as:
    \begin{equation}
        \lambda_{dyn}(t) = \lambda_{base} \cdot \exp\left( -\gamma \cdot \text{JSD}(P_{raw} || P_{cf}) \right).
    \end{equation}
    When JSD is low, it indicates that Agent L is dominating and the model is insensitive to visual changes. Consequently, $\lambda_{dyn}$ surges exponentially, imposing a heavy penalty on the hallucinated tokens in $\mathcal{L}_{cf}$.

    \item \textbf{Static Reward Term ($\beta$):} The term $+\beta \cdot \mathcal{L}_{iso}$ explicitly rewards tokens that align with the rigid visual subspace $\mathbf{V}_{iso}$. This numerically compensates for the payoff of Agent V, ensuring that factual visual details are prioritized.
\end{enumerate}

By sampling the next token $y_t$ from this re-balanced $\mathcal{L}_{final}$, ACE forces the system to shift from a False Equilibrium to a Factual Equilibrium, where the generation is strictly aligned with visual information. For implementation details, please refer to the pseudocode provided in Appendix~\ref{alg:ace_framework}. 

We empirically validate our physical propositions in Appendix E, demonstrating VSB instability through attention heatmap volatility and mapping the stream decoupling of DVS and CIS via color-coded generation trajectories.

\begin{table*}[t]
\centering
\caption{\textbf{Evaluation results on the CHAIR benchmark across four MLLMs.} We report sentence-level ($C_S$) and instance-level ($C_I$) scores, where lower values indicate better performance. The best results are highlighted in \textbf{bold}.}
\label{tab2}
\begin{tabular}{clcclcclcclcc}
\hline
Method             &  & \multicolumn{2}{c}{Shikra}    &  & \multicolumn{2}{c}{LLaVA-1.5} &  & \multicolumn{2}{c}{InstructBLIP} &  & \multicolumn{2}{c}{LLaVA-NeXT} \\ \cline{1-1} \cline{3-4} \cline{6-7} \cline{9-10} \cline{12-13} 
CHAIR              &  & $C_S$$\downarrow$        & $C_I$$\downarrow$        &  & $C_S$$\downarrow$         & $C_I$$\downarrow$      &  & $C_S$$\downarrow$          & $C_I$$\downarrow$         &  & $C_S$$\downarrow$         & $C_I$$\downarrow$        \\ \hline
Greedy             &  & 46.8          & 13.8          &  & 48.3           & 14.1         &  & 50.9            & 13.2           &  & 40.7           & 12.1          \\
VCD                &  & 47.1          & 13.6          &  & 45.5           & 12.8         &  & 43.2            & 13.9           &  & 38.3           & 11.4          \\
OPERA              &  & 45.1          & 12.7          &  & 44.1           & 12.2         &  & 45.5            & 12.3           &  & 37.6           & 10.2          \\
SID                &  & 43.9          & 12.5          &  & 44.5           & 12.7         &  & 43.5            & 12.1           &  & 37.9           & 11.5          \\
\rowcolor[HTML]{C0C0C0} 
\textbf{ACE(ours)} &  & \textbf{41.1} & \textbf{11.4} &  & \textbf{40.4}  & \textbf{9.6} &  & \textbf{41.7}   & \textbf{10.8}  &  & \textbf{35.4}  & \textbf{10.6} \\ \hline
\end{tabular}
\vskip -0.1in
\end{table*}

% Please add the following required packages to your document preamble:
% \usepackage[table,xcdraw]{xcolor}
% Beamer presentation requires \usepackage{colortbl} instead of \usepackage[table,xcdraw]{xcolor}

\section{Experiments}

This section systematically evaluates the hallucination mitigation capability of ACE across multiple MLLMs. The results show that ACE achieves overall competitive performance on mainstream evaluation metrics. Moreover, as a training-free test-time method, ACE exhibits strong generality and scalability, making it well suited for practical deployment scenarios.

\subsection{Experiment Setup}

\textbf{Models.} To verify the broad applicability of the proposed method across different MLLMs architectures, we apply ACE to several widely used models and conduct systematic evaluations, including \textbf{Shikra} \cite{chen2023shikra}, \textbf{LLaVA-1.5(7B)} \cite{liu2024improved}, \textbf{LLaVA-NeXT(8B)} \cite{li2024llava}, and \textbf{InstructBLIP} \cite{dai2023instructblip}. In addition, we employ Hyper-SDXL \cite{ren2024hyper} to generate counter-commonsense patches, which are used to introduce counter-commonsense perturbations during the decoding process.

\textbf{Evaluations.} We evaluate hallucination mitigation using three standard benchmarks: \textbf{POPE} \cite{li-etal-2023-evaluating} for object existence verification across three splits, \textbf{CHAIR} \cite{rohrbach-etal-2018-object} for captioning fidelity ($\text{CHAIR}_S$, $\text{CHAIR}_I$), and the \textbf{MME} \cite{fu2025mme} hallucination subset for fine-grained attribute consistency (see Appendix~\ref{app:metric_details} for full definitions).

\begin{table}[t]
\centering
\caption{\textbf{Inference latency comparison.} We evaluate the decoding speed (ms/token) on LLaVA-1.5. The values in parentheses denote the relative latency normalized to Greedy decoding. }
\label{tab3}
\begin{tabularx}{0.95\columnwidth}{>{\centering\arraybackslash}X >{\centering\arraybackslash}c} 
\hline
\textbf{Method} & \textbf{Inference Latency (ms/token)} \\ \hline
Greedy      & 19.8 (×1.0)       \\
VCD         & 46.4 (×2.34)      \\
OPERA       & 89.5 (×4.92)      \\
SID         & 78.6 (×3.97)      \\
\textbf{ACE (ours)}  & \textbf{44.4 (×2.24)}      \\ \hline
\end{tabularx}

\vskip -0.1in % 可选：微调图片与下文的间距
\end{table}

\textbf{Baseline. }In addition to comparing with the original MLLM sampling strategy, we include several representative hallucination mitigation methods as baselines:
\begin{itemize}[noitemsep, topsep=0pt, parsep=0pt, leftmargin=*]
 \item \textbf{Greedy Decoding:} Selects the token with the highest posterior probability at each decoding step.
    \item \textbf{VCD} \cite{leng2024mitigating}: Mitigates hallucinations by performing contrastive decoding against randomly perturbed image inputs.
    \item \textbf{OPERA} \cite{huang2024opera}: Penalizes logits to alleviate over-confidence during beam-search decoding, thereby adjusting the token selection strategy.
    \item \textbf{SID:} \cite{ICLR2025_3cc87f2b} Adaptively filters visual tokens and applies contrastive decoding by explicitly amplifying potential hallucination signals.
    % \item \textbf{SID} \cite{s}: Enhances the attention weights of image tokens during decoding and subtracts the logits of pure-text outputs, thereby increasing visual involvement and reducing hallucinated generations.

\end{itemize}
\paragraph{Hyperparameter Setting.} 
To ensure robust reproducibility, the hyperparameters for ACE are fixed across all evaluated tasks. 
Key settings include an adversarial library size $K=500$ and patch count $N=2$ for the CIS module. 
For the DVS, we set the soft-gating temperature $\kappa=20$ and the rigidity threshold $\tau=0.6$. 
Regarding the hierarchical intervention, the mid-layer injection ratio $\alpha$ is set to $0.2$ specifically targeting transformer layers 12--24. 
Finally, for the dynamic game decoding, we utilize a base penalty coefficient $\lambda_{base}=1.0$, a sensitivity control factor $\gamma=100$, and a static visual reward weight $\beta=0.4$. All experiments were conducted using PyTorch 2.0.1 and CUDA 11.7. Model inference and evaluation were performed on a single NVIDIA A800 GPU (80GB). More details are in Appendix ~\ref{app:hyperparams}.

% Please add the following required packages to your document preamble:
% \usepackage[table,xcdraw]{xcolor}
% Beamer presentation requires \usepackage{colortbl} instead of \usepackage[table,xcdraw]{xcolor}

\subsection{Results}
% \textbf{Results on POPE.}
% POPE is designed to strictly evaluate the object-level grounding capabilities of MLLMs by testing their ability to verify visual content via yes/no questions. As presented in Table. ~\ref{tab1}, ACE demonstrates consistent superiority across four diverse model architectures and all three evaluation settings. Notably, under the most challenging Adversarial setting—specifically designed to trigger language priors using high-frequency co-occurring objects—ACE achieves the most significant performance gains. For instance, with LLaVA-NeXT, ACE boosts accuracy from 80.12\% to 85.68\%, while also delivering a 4.5\% absolute gain on LLaVA-1.5. This phenomenon strongly corroborates our hypothesis: ACE effectively disrupts the False Equilibrium dominated by narrative inertia, thereby restoring decoding authority to the visual evidence. Compared to existing intervention methods such as SID and OPERA, ACE not only maintains a lead in accuracy but also exhibits a comprehensive advantage in F1 scores, which is particularly telling. This indicates that ACE does not merely trade accuracy by naively increasing the probability of negative responses ("No"). Instead, through its dynamic game-theoretic mechanism, it preserves perceptual sensitivity to ground-truth objects while mitigating hallucinations, thus achieving more trustworthy and balanced predictions.
\textbf{Results on POPE.}
POPE evaluates object-level grounding capabilities via yes/no questions. As shown in Table~\ref{tab1}, ACE demonstrates consistent superiority across all models and settings. Notably, under the challenging Adversarial setting, ACE achieves the most significant gains: for LLaVA-NeXT, accuracy boosts from 80.12\% to 85.68\%, alongside a 4.5\% gain on LLaVA-1.5. This strongly corroborates that ACE effectively disrupts the False Equilibrium driven by narrative inertia. Furthermore, compared to baselines like SID and OPERA, ACE maintains a comprehensive advantage in F1 scores. This indicates that ACE does not naively increase negative responses (``No''), but instead preserves perceptual sensitivity through its dynamic game mechanism, achieving trustworthy and balanced predictions.

\textbf{Results on CHAIR.}
CHAIR serves as a widely adopted benchmark for evaluating hallucination in MLLMs during image captioning, comparing generated descriptions against ground-truth object annotations. It captures object-level precision via two metrics: CHAIR$_S$ (sentence-level) and CHAIR$_I$ (instance-level). As shown in the table. ~\ref{tab2}, ACE consistently outperforms existing decoding strategies across all mainstream model variants. ACE achieves the lowest CHAIR$_S$ and CHAIR$_I$ scores throughout, demonstrating its effectiveness in mitigating hallucinations during long-text generation. Compared to second-best methods such as SID and OPERA, ACE exhibits superior robustness. This advantage is particularly pronounced on LLaVA-1.5, where ACE substantially reduces CHAIR$_S$ from 48.3 (Greedy) to 40.4, and suppresses CHAIR$_I$ to 9.6, significantly outperforming OPERA’s 12.2. In contrast to the performance fluctuations observed in other methods across different architectures (e.g., VCD shows limited improvement on InstructBLIP), ACE demonstrates remarkable cross-architecture stability. This finding strongly supports our hypothesis: hallucinations in long-form descriptions often stem from "narrative inertia," where the language model gradually detaches from visual constraints during generation. By introducing a dynamic game-theoretic mechanism at each decoding step, ACE enforces alignment between language priors and visual information, thereby generating descriptions that are faithful to the image content without compromising descriptive richness.

\begin{figure}[t]  % 单栏用figure，双栏跨页用figure*
\vskip 0.2in
  \centering  % 居中对齐
  % 子图(a)：占单栏宽度的48%（留2%间距）
  \subfloat[LLaVA-1.5]{% 括号内是子图标签，括号外是标题
    \includegraphics[width=0.45\columnwidth]{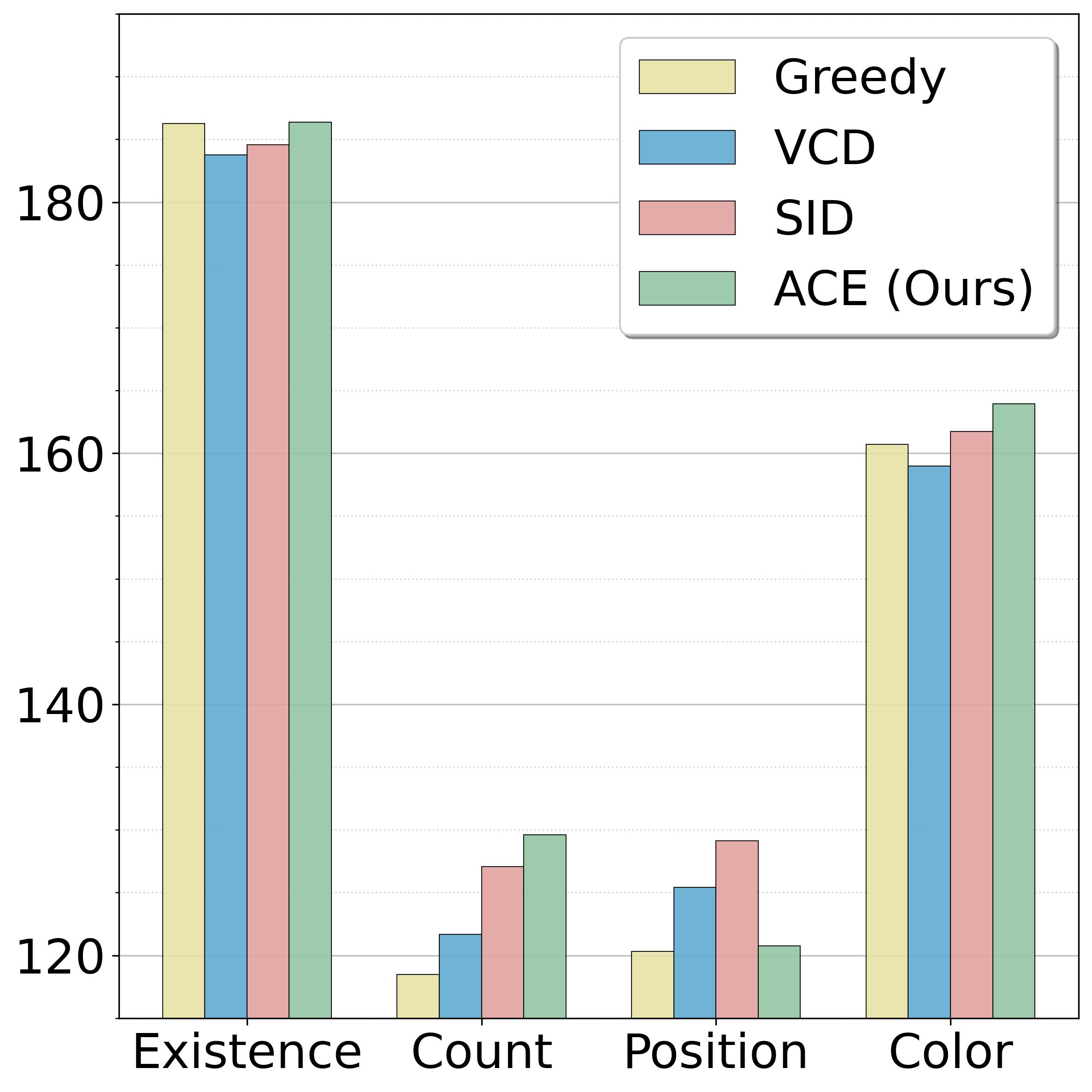}
  }
  \hfill  % 子图间自动填充间距（无多余空格）
  % 子图(b)：占单栏宽度的48%
  \subfloat[InstructBLIP]{%
    \includegraphics[width=0.45\columnwidth]{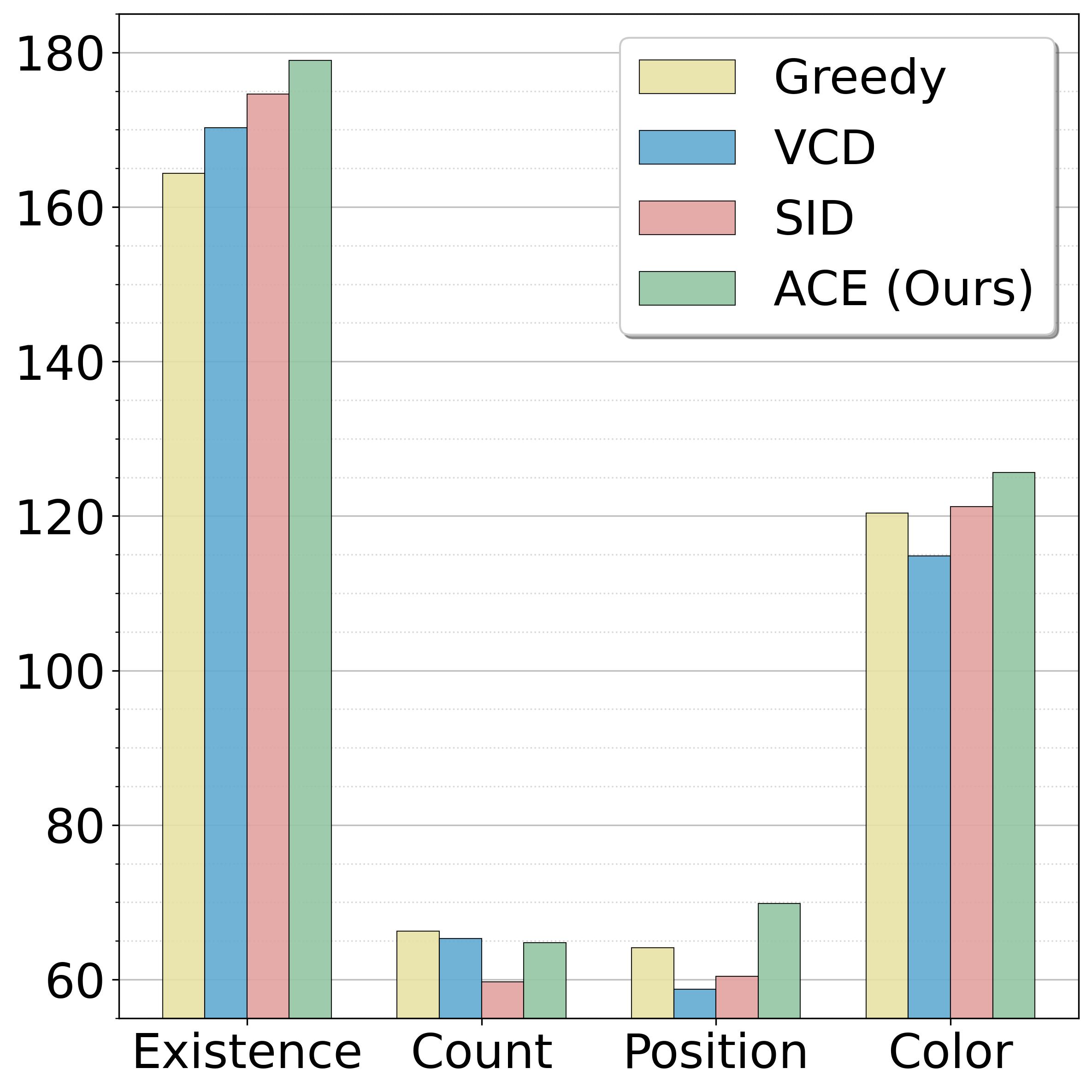}
  }
  % 整体图标题（可自定义）
  \caption{\textbf{Performance comparison on the MME hallucination subset.} 
We evaluate ACE against baseline decoding strategies on (a) LLaVA-1.5 and (b) InstructBLIP across four categories: Existence, Count, Position, and Color. }
\label{fig3}
  \vskip -0.2in
\end{figure}

\textbf{Results on MME.}
We conduct comprehensive experiments on the MME hallucination subset, which includes four types of hallucinations designed to assess the overall capability of MLLMs at the object level (i.e., Existence), attribute level (i.e., Count and Color), and relation level (i.e., Position). As illustrated in Fig.~\ref{fig3}, the proposed ACE consistently outperforms both the Greedy decoding and SOTA strategies across diverse hallucination categories and MLLM architectures. Notably, ACE demonstrates substantial gains in fine-grained attribute perception. On LLaVA-1.5, ACE boosts the Count score from 112.49 to 129.62, surpassing SID (127.06), and achieves the highest performance in Color recognition (163.94). On InstructBLIP, ACE achieves a remarkable breakthrough in Existence verification (improving from 164.37 to 178.98), significantly outperforming all baselines. This underscores ACE’s effectiveness in addressing a broader range of multimodal hallucinations beyond simple object existence. By restoring the decoding authority to visual information via its game-theoretic mechanism, ACE enhances the model's general perception capabilities, particularly in capturing subtle visual details such as quantity and color.

\begin{figure}[t]
  \vskip 0.2in
  \begin{center}
    % width=\textwidth 适配双栏总宽度（替代原单栏的 \columnwidth）
    % \centerline{\includegraphics[width=\textwidth]{icml_numpapers}}
    \includegraphics[width=0.9\columnwidth]{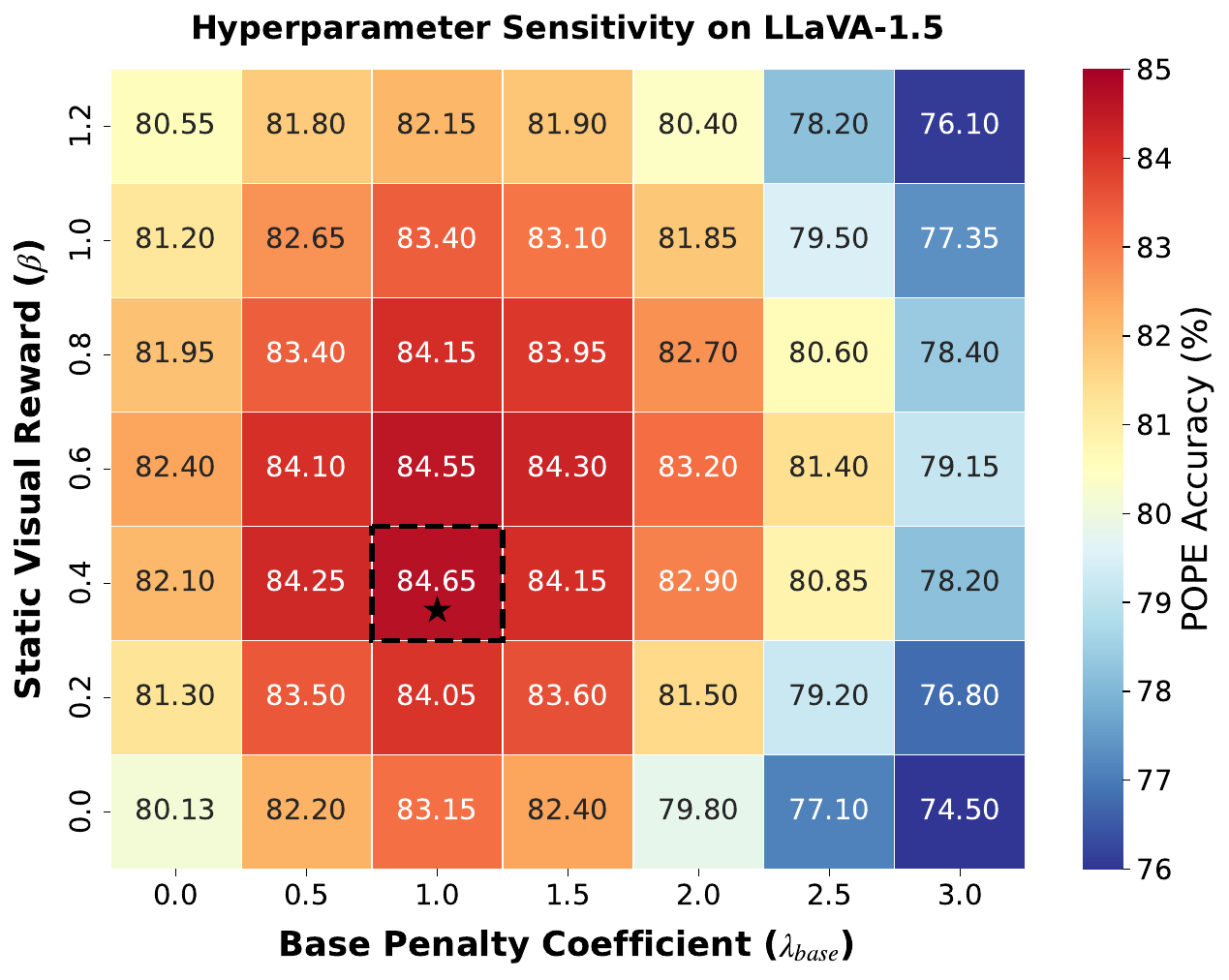} 
    \caption{\textbf{Hyperparameter sensitivity analysis of $\lambda_{base}$ and $\beta$ on LLaVA-1.5.} 
The heatmap visualizes POPE accuracy under varying combinations of the penalty coefficient $\lambda_{base}$ (suppressing Agent L) and the reward weight $\beta$ (reinforcing Agent V). 
A distinct ``Ridge of Equilibrium'' emerges, with the model achieving peak performance (84.65\%) at the optimal game point ($\lambda_{base}=1.0, \beta=0.4$). }
\label{fig:hyperparameter_heatmap}
    \label{fig4}
  \end{center}
  \vskip -0.3in 
\end{figure}

\textbf{Latency analysis.} Table~\ref{tab3} evaluates the inference efficiency of ACE on LLaVA-1.5 (7B). While competitive methods like OPERA ($4.92\times$) and SID ($3.97\times$) are hampered by iterative backtracking, ACE maintains a streamlined $2.24\times$ latency, notably outperforming even the dual-stream VCD ($2.34\times$). This efficiency is primarily attributed to our offline feature caching strategy. Unlike VCD, which necessitates a redundant and costly forward pass through the vision encoder for the distorted image, our CIS and DVS modules perform lightweight retrieval from a pre-computed library $\mathcal{B}$. By substituting expensive image re-encoding with simple matrix multiplications and mid-layer injections, ACE effectively bypasses the primary computational bottleneck of MLLMs. Consequently, the framework achieves an optimal equilibrium between factual grounding and inference speed, demonstrating strong potential for real-world deployment where low-latency response is critical.

\textbf{How to determine the optimal operating point for game equilibrium? }
In our dynamic game framework, the synergistic interaction between the penalty coefficient $\lambda_{base}$ (targeting Agent L) and the reward weight $\beta$ (targeting Agent V) is critical for hallucination mitigation. 
We investigate the performance under different parameter combinations via grid search, as illustrated in Fig.~\ref{fig4}. As shown in the heatmap, the performance distribution exhibits a distinct ``Ridge of Equilibrium'' rather than a simple linear trend. 
The optimal performance consistently converges around $(\lambda_{base}=1.0, \beta=0.4)$, suggesting that this configuration strikes the best balance between suppressing narrative inertia and preserving visual faithfulness.
Crucially, we observe a precipitous drop in accuracy when $\lambda_{base}$ exceeds $2.0$, regardless of the compensation from $\beta$. 
This suggests that excessive penalization leads to severe ``over-correction,'' where the model's reasoning capability collapses under the strict constraints. 
Conversely, an overly high reward ($\beta > 1.0$) yields diminishing returns by introducing visual noise. 
These empirical findings strongly support our theory: neither unilateral suppression nor reward suffices; only a precise \textbf{bidirectional synergy} can maintain the Nash Equilibrium of the system.

\begin{figure}[t]
  \vskip 0.2in
  \begin{center}
    % width=\textwidth 适配双栏总宽度（替代原单栏的 \columnwidth）
    % \centerline{\includegraphics[width=\textwidth]{icml_numpapers}}
    \includegraphics[width=0.95\columnwidth]{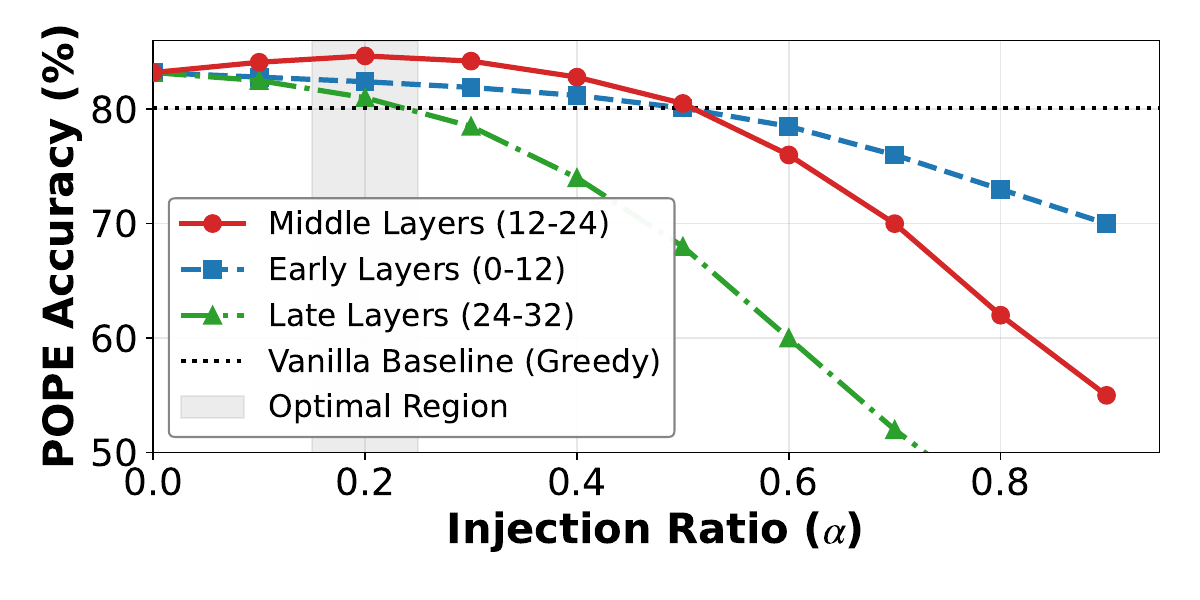} 
    \caption{\textbf{Impact of intervention depth and injection ratio $\alpha$ on LLaVA-1.5.} Mid-layer rectification serves as the optimal intervention window, whereas early and late injections suffer from feature mismatch and semantic conflict, respectively.}
    \label{fig5}
  \end{center}
  \vskip -0.3in % 可选：微调图片与下文的间距
\end{figure}

% \textbf{Impact of Intervention Depth and Intensity.} To assess the specific contribution of Mid-layer Hidden State Rectification within our hierarchical framework, we analyze the impact of injection depth and the mixing ratio $\alpha$ in Fig.~\ref{fig5}. 
% Notably, even at $\alpha=0$ (i.e., retaining only the final-layer dynamic game), the framework maintains a robust performance of 83.20\%, significantly outperforming the greedy baseline (80.13\%), which highlights the efficacy of CIS and DVS in reshaping the Payoff Matrix at the logit level. Building on this, injection within the Middle Layers (12--24) yields the optimal trajectory, peaking at \textbf{84.65\%} with $\alpha=0.2$. 
% This trend suggests that reinforcing Agent V during the formative stage of the VSB provides critical marginal gains by preventing visual information from being silenced by \textit{Narrative Inertia}. In contrast, early-layer injection results in a monotonic decline, likely due to the feature-level mismatch between raw shallow features and the high-level semantic signals carried by the DVS. 
% More critically, late-layer injection exhibits extreme fragility; performance collapses precipitously when $\alpha > 0.2$, as excessive visual intrusion into the deep layers dominated by Agent L causes severe \textit{semantic conflict} and attention failure. 
% These findings validate that fixing $\alpha=0.2$ within the middle layers offers the precise intervention required to achieve \textit{Dynamic Equilibrium} within the VSB.

\textbf{Impact of Intervention Depth and Intensity.} To assess Mid-layer Hidden State Rectification, we analyze the impact of injection depth and mixing ratio $\alpha$ in Fig.~\ref{fig5}. Notably, even at $\alpha=0$ (final-layer intervention only), the framework achieves 83.20\%, significantly outperforming the baseline (80.13\%) and highlighting the efficacy of reshaping the Payoff Matrix at the logit level. Optimal performance peaks at \textbf{84.65\%} with $\alpha=0.2$ in the Middle Layers (12--24). This suggests that reinforcing Agent V during the VSB's formative stage is critical to counteract \textit{Narrative Inertia}. Conversely, early-layer injection declines due to feature-level mismatch, while late-layer injection proves fragile; performance collapses when $\alpha > 0.2$ as excessive visual intrusion triggers severe \textit{semantic conflict} with Agent L. These findings validate that fixing $\alpha=0.2$ in the middle layers provides the precise intervention required for \textit{Dynamic Equilibrium}.

% \textbf{Ablation Study.}

% Table.~\ref{tab4} dissects the independent contribution of each module within the ACE framework. 
% A critical observation arises when removing the DVS (w/o DVS): while Accuracy improves over the baseline, the F1-score exhibits an anomalous drop ($80.96\% \to 80.25\%$). 
% This empirical evidence suggests that relying solely on the CIS to penalize Agent L forces the model into an over-conservative state, sacrificing recall to suppress hallucinations; thus, the positive visual anchor provided by DVS is indispensable for maintaining robustness. 
% Conversely, removing the CIS (w/o CIS) yields a high F1-score ($82.95\%$) by reinforcing Agent V, yet fails to achieve the lowest hallucination rates (CHAIR), indicating that without a ``Negative Probe'' to disrupt Narrative Inertia, reward mechanisms alone cannot fully counteract stubborn linguistic priors. 
% More critically, even with both CIS and DVS active at the final layer, the exclusion of mid-layer injection (w/o Mid-Layer) results in suboptimal performance (lagging by $0.7\%$ in Acc). 
% This finding strongly corroborates our VSB theory: visual information is structurally silenced within intermediate layers, and mere logit-level calibration is insufficient. 
% Therefore, Mid-Layer Rectification is essential to resurrect Agent V at the representational level, allowing the full ACE framework to achieve a true Dynamic Equilibrium.

\textbf{Ablation Study.} Table.~\ref{tab4} dissects the contribution of each ACE module. Removing DVS (\textit{w/o DVS}) improves accuracy over the baseline but causes an anomalous F1-score drop (80.96\%$\rightarrow$80.25\%). This suggests that solely penalizing Agent L via CIS forces the model into an over-conservative state; thus, the DVS positive anchor is indispensable for robustness. Conversely, omitting CIS (\textit{w/o CIS}) yields a high F1-score (82.95\%) but fails to achieve the lowest hallucination rates, indicating that without a ``Negative Probe'' to disrupt Narrative Inertia, reward mechanisms alone cannot fully counteract stubborn priors. Crucially, even with both modules active, excluding mid-layer injection (\textit{w/o Mid-Layer}) results in suboptimal performance (0.7\% Acc lag). This corroborates our VSB theory: visual information is structurally silenced within intermediate layers, making mere logit-level calibration insufficient. Therefore, Mid-Layer Rectification is essential to resurrect Agent V at the representational level, enabling ACE to achieve a true Dynamic Equilibrium.

\section{Conclusions, Limitations and Future Work}

\textbf{Conclusions.} In this paper, we propose Adversarial Counter-Commonsense Equilibrium (ACE), a novel training-free decoding framework. 
ACE operates by introducing counter-commonsense adversarial patches to induce visual-language response disparities, implementing synergistic intervention at hierarchical levels of the decoding process. 
Grounded in our analysis that decoding game disequilibrium results in visual information being structurally silenced within the Visual Semantic Buffers, ACE exploits the characteristic that hallucinations are sensitive to perturbations while authentic visual features remain stable. 
Experimental results demonstrate that ACE achieves consistent improvements across multiple benchmarks, striking a favorable balance between inference efficiency and hallucination suppression, while offering valuable insights into the information flow dynamics within MLLMs.

\textbf{Limitations and Future Work.} While robust, ACE's fixed hyperparameters may cause conservative generation in ambiguous contexts. Future work will introduce a lightweight gating network for adaptive intervention and leverage ACE-generated preference data for self-supervised alignment (e.g., DPO) to internalize the equilibrium mechanism, enhancing intrinsic model robustness.

\begin{table}[t]
\centering
\caption{\textbf{Ablation study on the individual contribution of each component in ACE.} We conduct experiments on LLaVA-1.5 to validate the effectiveness of the CIS, the DVS, and the Mid-Layer Rectification mechanism. Results are reported on the POPE (Adversarial split) and CHAIR benchmarks ($C_S$ and $C_I$), confirming the necessity of hierarchical synergistic intervention.}
\label{tab4}
\resizebox{\linewidth}{!}{
\begin{tabular}{l | c c c | c c | c c}
\toprule
\multirow{2}{*}{\textbf{Method}} & \multicolumn{3}{c|}{\textbf{Components}} & \multicolumn{2}{c|}{\textbf{POPE (Adv)}} & \multicolumn{2}{c}{\textbf{CHAIR}} \\
 & \textbf{CIS} & \textbf{DVS} & \makecell{\textbf{Mid-Layer}\\\textbf{Rectification}} & \textbf{Acc} $\uparrow$ & \textbf{F1} $\uparrow$ & \textbf{$C_S$} $\downarrow$ & \textbf{$C_I$} $\downarrow$ \\
\midrule
Baseline(LLaVA-1.5) & \ding{55} & \ding{55} & \ding{55} & 80.13 & 80.96 & 48.3 & 14.1 \\
\midrule
w/o DVS (Penalty Only) & \ding{51} & \ding{55} & \ding{55} & 81.85 & 80.25 & 45.2 & 12.9 \\
w/o CIS (Reward Only) & \ding{55} & \ding{51} & \ding{51} & 83.10 & \textbf{82.95} & 42.8 & 11.2 \\
w/o Mid-Layer (Late Only) & \ding{51} & \ding{51} & \ding{55} & 83.95 & 82.20 & 41.5 & 10.5 \\
\midrule
\rowcolor[HTML]{EFEFEF}
\textbf{ACE (Full)} & \ding{51} & \ding{51} & \ding{51} & \textbf{84.65} & 82.59 & \textbf{40.4} & \textbf{9.6} \\
\bottomrule
\end{tabular}
}
\vskip -0.1in % 可选：微调图片与下文的间距
\end{table}

% Acknowledgements should only appear in the accepted version.
% \section*{Acknowledgements}

% \textbf{Do not} include acknowledgements in the initial version of the paper
% submitted for blind review.

% If a paper is accepted, the final camera-ready version can (and usually should)
% include acknowledgements.  Such acknowledgements should be placed at the end of
% the section, in an unnumbered section that does not count towards the paper
% page limit. Typically, this will include thanks to reviewers who gave useful
% comments, to colleagues who contributed to the ideas, and to funding agencies
% and corporate sponsors that provided financial support.
\newpage

% This paper presents work whose goal is to advance the field of Machine
% Learning. There are many potential societal consequences of our work, none
% which we feel must be specifically highlighted here.

% \section*{Impact Statement}
% This paper introduces Adversarial Counter-Commonsense Equilibrium, a training-free framework designed to mitigate hallucinations in Multimodal Large Language Models. There are many potential societal consequences of our work, none which we feel must be specifically highlighted here.

% In the unusual situation where you want a paper to appear in the
% references without citing it in the main text, use \nocite
\nocite{langley00}

\bibliography{main}

@inproceedings{park2025second,
  title     = {SECOND: Mitigating Perceptual Hallucination in Vision-Language Models via Selective and Contrastive Decoding},
  author    = {Park, Woohyeon and Kim, Woojin and Kim, Jaeik and Do, Jaeyoung},
  booktitle = {Proceedings of the 42nd International Conference on Machine Learning (ICML)},
  year      = {2025},
  series    = {Proceedings of Machine Learning Research},
  publisher = {PMLR}
}

@inproceedings{woo2025don,
  title={Don’t miss the forest for the trees: Attentional vision calibration for large vision language models},
  author={Woo, Sangmin and Kim, Donguk and Jang, Jaehyuk and Choi, Yubin and Kim, Changick},
  booktitle={Findings of the Association for Computational Linguistics: ACL 2025},
  pages={1927--1951},
  year={2025}
}

@inproceedings{che2025hallucinatory,
  title={Hallucinatory Image Tokens: A Training-free EAZY Approach to Detecting and Mitigating Object Hallucinations in LVLMs},
  author={Che, Liwei and Liu, Tony Qingze and Jia, Jing and Qin, Weiyi and Tang, Ruixiang and Pavlovic, Vladimir},
  booktitle={Proceedings of the IEEE/CVF International Conference on Computer Vision},
  pages={21635--21644},
  year={2025}
}

@article{li2025mitigating,
  title={Mitigating hallucination for large vision language model by inter-modality correlation calibration decoding},
  author={Li, Jiaming and Zhang, Jiacheng and Jie, Zequn and Ma, Lin and Li, Guanbin},
  journal={arXiv preprint arXiv:2501.01926},
  year={2025}
}

@inproceedings{rohrbach-etal-2018-object,
    title = "Object Hallucination in Image Captioning",
    author = "Rohrbach, Anna  and
      Hendricks, Lisa Anne  and
      Burns, Kaylee  and
      Darrell, Trevor  and
      Saenko, Kate",
    editor = "Riloff, Ellen  and
      Chiang, David  and
      Hockenmaier, Julia  and
      Tsujii, Jun{'}ichi",
    booktitle = "Proceedings of the 2018 Conference on Empirical Methods in Natural Language Processing",
    month = oct # "-" # nov,
    year = "2018",
    address = "Brussels, Belgium",
    publisher = "Association for Computational Linguistics",
    doi = "10.18653/v1/D18-1437",
    pages = "4035--4045"
}

@inproceedings{li-etal-2023-evaluating,
    title = "Evaluating Object Hallucination in Large Vision-Language Models",
    author = "Li, Yifan  and
      Du, Yifan  and
      Zhou, Kun  and
      Wang, Jinpeng  and
      Zhao, Xin  and
      Wen, Ji-Rong",
    editor = "Bouamor, Houda  and
      Pino, Juan  and
      Bali, Kalika",
    booktitle = "Proceedings of the 2023 Conference on Empirical Methods in Natural Language Processing",
    month = dec,
    year = "2023",
    address = "Singapore",
    publisher = "Association for Computational Linguistics",
    doi = "10.18653/v1/2023.emnlp-main.20",
    pages = "292--305"
}

@inproceedings{fu2025mme,
  title={Mme: A comprehensive evaluation benchmark for multimodal large language models},
  author={Fu, Chaoyou and Chen, Peixian and Shen, Yunhang and Qin, Yulei and Zhang, Mengdan and Lin, Xu and Yang, Jinrui and Zheng, Xiawu and Li, Ke and Sun, Xing and others},
  booktitle={The Thirty-ninth Annual Conference on Neural Information Processing Systems Datasets and Benchmarks Track},
  year={2025}
}

@inproceedings{back2025watermarking,
  title={Watermarking for Factuality: Guiding Vision-Language Models Toward Truth via Tri-layer Contrastive Decoding},
  author={Back, Kyungryul and Park, Seongbeom and Kim, Milim and Kwon, Mincheol and Lee, SangHyeok and Lee, Hyunyoung and Cho, Junhee and Park, Seunghyun and Kim, Jinkyu},
  booktitle={Findings of the Association for Computational Linguistics: EMNLP 2025},
  pages={8371--8387},
  year={2025}
}

@article{cho2025revisitseediscloselanguage,
      title={Revisit What You See: Disclose Language Prior in Vision Tokens for LVLM Decoding}, 
      author={Beomsik Cho and Jaehyung Kim},
      year={2025},
      journal={arXiv preprint arXiv:2506.09522},
}

@inproceedings{leng2024mitigating,
  title={Mitigating object hallucinations in large vision-language models through visual contrastive decoding},
  author={Leng, Sicong and Zhang, Hang and Chen, Guanzheng and Li, Xin and Lu, Shijian and Miao, Chunyan and Bing, Lidong},
  booktitle={Proceedings of the IEEE/CVF Conference on Computer Vision and Pattern Recognition},
  pages={13872--13882},
  year={2024}
}

@inproceedings{
kim2024instructive,
title={Instructive Decoding: Instruction-Tuned Large Language Models are Self-Refiner from Noisy Instructions},
author={Taehyeon Kim and Joonkee Kim and Gihun Lee and Se-Young Yun},
booktitle={The Twelfth International Conference on Learning Representations},
year={2024},
}

@inproceedings{wang-etal-2024-mitigating,
    title = "Mitigating Hallucinations in Large Vision-Language Models with Instruction Contrastive Decoding",
    author = "Wang, Xintong  and
      Pan, Jingheng  and
      Ding, Liang  and
      Biemann, Chris",
    editor = "Ku, Lun-Wei  and
      Martins, Andre  and
      Srikumar, Vivek",
    booktitle = "Findings of the Association for Computational Linguistics: ACL 2024",
    month = aug,
    year = "2024",
    address = "Bangkok, Thailand",
    publisher = "Association for Computational Linguistics",
    doi = "10.18653/v1/2024.findings-acl.937",
    pages = "15840--15853",
}

@inproceedings{ICLR2025_3cc87f2b,
 author = {Huo, Fushuo and Xu, Wenchao and Zhang, Zhong and Wang, Haozhao and Chen, Zhicheng and Zhao, Peilin},
 booktitle = {International Conference on Representation Learning},
 editor = {Y. Yue and A. Garg and N. Peng and F. Sha and R. Yu},
 pages = {24272--24295},
 title = {Self-Introspective Decoding: Alleviating Hallucinations for Large Vision-Language Models},
 volume = {2025},
 year = {2025}
}

@inproceedings{ICLR2024_edc36117,
 author = {Chuang, Yung-Sung and Xie, Yujia and Luo, Hongyin and Kim, Yoon and Glass, James R and He, Pengcheng},
 booktitle = {International Conference on Representation Learning},
 editor = {B. Kim and Y. Yue and S. Chaudhuri and K. Fragkiadaki and M. Khan and Y. Sun},
 pages = {54158--54183},
 title = {DoLa: Decoding by Contrasting Layers Improves Factuality in Large Language Models},
 volume = {2024},
 year = {2024}
}

@article{zhao2024mitigating,
  title={Mitigating object hallucination in large vision-language models via image-grounded guidance},
  author={Zhao, Linxi and Deng, Yihe and Zhang, Weitong and Gu, Quanquan},
  journal={arXiv preprint arXiv:2402.08680},
  year={2024}
}

@inproceedings{min2025mitigating,
  title={Mitigating hallucinations in large vision-language models via summary-guided decoding},
  author={Min, Kyungmin and Kim, Minbeom and Lee, Kang-il and Lee, Dongryeol and Jung, Kyomin},
  booktitle={Findings of the Association for Computational Linguistics: NAACL 2025},
  pages={4183--4198},
  year={2025}
}

@article{li2025hidden,
  title={The hidden life of tokens: Reducing hallucination of large vision-language models via visual information steering},
  author={Li, Zhuowei and Shi, Haizhou and Gao, Yunhe and Liu, Di and Wang, Zhenting and Chen, Yuxiao and Liu, Ting and Zhao, Long and Wang, Hao and Metaxas, Dimitris N},
  journal={arXiv preprint arXiv:2502.03628},
  year={2025}
}

@article{zoulook,
  title={Look Twice Before You Answer: Memory-Space Visual Retracing for Hallucination Mitigation in Multimodal Large Language Models}, 
  author={Zou, Xin and Wang, Yizhou and Yan, Yibo and Lyu, Yuanhuiyi and Zheng, Kening and Huang, Sirui and Chen, Junkai and Jiang, Peijie and Liu, Jia and Tang, Chang and Hu, Xuming},
  journal={Forty-second International Conference on Machine Learning (ICML)},
  year={2025}
}

@inproceedings{ICLR2025_109cf25c,
 author = {Zhang, Ce and Wan, Zifu and Kan, Zhehan and Ma, Martin Q. and Stepputtis, Simon and Ramanan, Deva  and Salakhutdinov, Russ and Morency, Louis-Philippe and Sycara, Katia and Xie, Yaqi},
 booktitle = {International Conference on Representation Learning},
 editor = {Y. Yue and A. Garg and N. Peng and F. Sha and R. Yu},
 pages = {5494--5524},
 title = {Self-Correcting Decoding with Generative Feedback for Mitigating Hallucinations in Large Vision-Language Models},
 volume = {2025},
 year = {2025}
}

@InProceedings{pmlr-v267-luo25b,
  title = 	 {Probing Visual Language Priors in {VLM}s},
  author =       {Luo, Tiange and Cao, Ang and Lee, Gunhee and Johnson, Justin and Lee, Honglak},
  booktitle = 	 {Proceedings of the 42nd International Conference on Machine Learning},
  pages = 	 {41120--41156},
  year = 	 {2025},
  editor = 	 {Singh, Aarti and Fazel, Maryam and Hsu, Daniel and Lacoste-Julien, Simon and Berkenkamp, Felix and Maharaj, Tegan and Wagstaff, Kiri and Zhu, Jerry},
  volume = 	 {267},
  series = 	 {Proceedings of Machine Learning Research},
  month = 	 {13--19 Jul},
  publisher =    {PMLR}
}

@article{lu2024deepseek,
  title={Deepseek-vl: towards real-world vision-language understanding},
  author={Lu, Haoyu and Liu, Wen and Zhang, Bo and Wang, Bingxuan and Dong, Kai and Liu, Bo and Sun, Jingxiang and Ren, Tongzheng and Li, Zhuoshu and Yang, Hao and others},
  journal={arXiv preprint arXiv:2403.05525},
  year={2024}
}

@inproceedings{lai2024lisa,
  title={Lisa: Reasoning segmentation via large language model},
  author={Lai, Xin and Tian, Zhuotao and Chen, Yukang and Li, Yanwei and Yuan, Yuhui and Liu, Shu and Jia, Jiaya},
  booktitle={Proceedings of the IEEE/CVF Conference on Computer Vision and Pattern Recognition},
  pages={9579--9589},
  year={2024}
}

@inproceedings{geng2024instructdiffusion,
  title={Instructdiffusion: A generalist modeling interface for vision tasks},
  author={Geng, Zigang and Yang, Binxin and Hang, Tiankai and Li, Chen and Gu, Shuyang and Zhang, Ting and Bao, Jianmin and Zhang, Zheng and Li, Houqiang and Hu, Han and others},
  booktitle={Proceedings of the IEEE/CVF Conference on computer vision and pattern recognition},
  pages={12709--12720},
  year={2024}
}

@article{bai2024hallucination,
  title={Hallucination of multimodal large language models: A survey},
  author={Bai, Zechen and Wang, Pichao and Xiao, Tianjun and He, Tong and Han, Zongbo and Zhang, Zheng and Shou, Mike Zheng},
  journal={arXiv preprint arXiv:2404.18930},
  year={2024}
}

@inproceedings{guan2024hallusionbench,
  title={Hallusionbench: an advanced diagnostic suite for entangled language hallucination and visual illusion in large vision-language models},
  author={Guan, Tianrui and Liu, Fuxiao and Wu, Xiyang and Xian, Ruiqi and Li, Zongxia and Liu, Xiaoyu and Wang, Xijun and Chen, Lichang and Huang, Furong and Yacoob, Yaser and others},
  booktitle={Proceedings of the IEEE/CVF Conference on Computer Vision and Pattern Recognition},
  pages={14375--14385},
  year={2024}
}

@inproceedings{chen2024towards,
  title={Towards injecting medical visual knowledge into multimodal llms at scale},
  author={Chen, Junying and Gui, Chi and Ouyang, Ruyi and Gao, Anningzhe and Chen, Shunian and Chen, Guiming Hardy and Wang, Xidong and Cai, Zhenyang and Ji, Ke and Wan, Xiang and others},
  booktitle={Proceedings of the 2024 conference on empirical methods in natural language processing},
  pages={7346--7370},
  year={2024}
}

@article{liu2024survey,
  title={A survey on hallucination in large vision-language models},
  author={Liu, Hanchao and Xue, Wenyuan and Chen, Yifei and Chen, Dapeng and Zhao, Xiutian and Wang, Ke and Hou, Liping and Li, Rongjun and Peng, Wei},
  journal={arXiv preprint arXiv:2402.00253},
  year={2024}
}

@article{liu2023visual,
  title={Visual instruction tuning},
  author={Liu, Haotian and Li, Chunyuan and Wu, Qingyang and Lee, Yong Jae},
  journal={Advances in neural information processing systems},
  volume={36},
  pages={34892--34916},
  year={2023}
}

@article{chen2023shikra,
  title={Shikra: Unleashing multimodal llm's referential dialogue magic},
  author={Chen, Keqin and Zhang, Zhao and Zeng, Weili and Zhang, Richong and Zhu, Feng and Zhao, Rui},
  journal={arXiv preprint arXiv:2306.15195},
  year={2023}
}

@article{chen2024far,
    title={How Far Are We to GPT-4V? Closing the Gap to Commercial Multimodal Models with Open-Source Suites},
    author={Chen, Zhe and Wang, Weiyun and Tian, Hao and Ye, Shenglong and Gao, Zhangwei and Cui, Erfei and Tong, Wenwen and Hu, Kongzhi and Luo, Jiapeng and Ma, Zheng and others},
    journal={arXiv preprint arXiv:2404.16821},
    year={2024}
  }

@inproceedings{chen2024internvl,
    title={Internvl: Scaling up vision foundation models and aligning for generic visual-linguistic tasks},
    author={Chen, Zhe and Wu, Jiannan and Wang, Wenhai and Su, Weijie and Chen, Guo and Xing, Sen and Zhong, Muyan and Zhang, Qinglong and Zhu, Xizhou and Lu, Lewei and others},
    booktitle={Proceedings of the IEEE/CVF Conference on Computer Vision and Pattern Recognition},
    pages={24185--24198},
    year={2024}
  }

@article{chen2024halc,
  title={Halc: Object hallucination reduction via adaptive focal-contrast decoding},
  author={Chen, Zhaorun and Zhao, Zhuokai and Luo, Hongyin and Yao, Huaxiu and Li, Bo and Zhou, Jiawei},
  journal={arXiv preprint arXiv:2403.00425},
  year={2024}
}

@inproceedings{favero2024multi,
  title={Multi-modal hallucination control by visual information grounding},
  author={Favero, Alessandro and Zancato, Luca and Trager, Matthew and Choudhary, Siddharth and Perera, Pramuditha and Achille, Alessandro and Swaminathan, Ashwin and Soatto, Stefano},
  booktitle={Proceedings of the IEEE/CVF Conference on Computer Vision and Pattern Recognition},
  pages={14303--14312},
  year={2024}
}

@inproceedings{gunjal2024detecting,
  title={Detecting and preventing hallucinations in large vision language models},
  author={Gunjal, Anisha and Yin, Jihan and Bas, Erhan},
  booktitle={Proceedings of the AAAI Conference on Artificial Intelligence},
  volume={38},
  number={16},
  pages={18135--18143},
  year={2024}
}

@inproceedings{kim2023exposing,
  title={Exposing and mitigating spurious correlations for cross-modal retrieval},
  author={Kim, Jae Myung and Koepke, A and Schmid, Cordelia and Akata, Zeynep},
  booktitle={Proceedings of the IEEE/CVF conference on computer vision and pattern recognition},
  pages={2585--2595},
  year={2023}
}

@inproceedings{ICLR2024_fc625e83,
 author = {Zhou, Yiyang and Cui, Chenhang and Yoon, Jaehong and Zhang, Linjun and Deng, Zhun and Finn, Chelsea and Bansal, Mohit and Yao, Huaxiu},
 booktitle = {International Conference on Representation Learning},
 editor = {B. Kim and Y. Yue and S. Chaudhuri and K. Fragkiadaki and M. Khan and Y. Sun},
 pages = {56969--56998},
 title = {Analyzing and Mitigating Object Hallucination in Large Vision-Language Models},
 volume = {2024},
 year = {2024}
}

@inproceedings{yu2024rlhf,
  title={Rlhf-v: Towards trustworthy mllms via behavior alignment from fine-grained correctional human feedback},
  author={Yu, Tianyu and Yao, Yuan and Zhang, Haoye and He, Taiwen and Han, Yifeng and Cui, Ganqu and Hu, Jinyi and Liu, Zhiyuan and Zheng, Hai-Tao and Sun, Maosong and others},
  booktitle={Proceedings of the IEEE/CVF Conference on Computer Vision and Pattern Recognition},
  pages={13807--13816},
  year={2024}
}

@inproceedings{ICLR2024_0b408293,
 author = {Darcet, Timoth\'{e}e and Oquab, Maxime and Mairal, Julien and Bojanowski, Piotr},
 booktitle = {International Conference on Representation Learning},
 editor = {B. Kim and Y. Yue and S. Chaudhuri and K. Fragkiadaki and M. Khan and Y. Sun},
 pages = {2632--2652},
 title = {Vision Transformers Need Registers},
 volume = {2024},
 year = {2024}
}

@article{ren2024hyper,
  title={Hyper-sd: Trajectory segmented consistency model for efficient image synthesis},
  author={Ren, Yuxi and Xia, Xin and Lu, Yanzuo and Zhang, Jiacheng and Wu, Jie and Xie, Pan and Wang, Xing and Xiao, Xuefeng},
  journal={Advances in Neural Information Processing Systems},
  volume={37},
  pages={117340--117362},
  year={2024}
}

@article{li2024llava,
  title={Llava-next-interleave: Tackling multi-image, video, and 3d in large multimodal models},
  author={Li, Feng and Zhang, Renrui and Zhang, Hao and Zhang, Yuanhan and Li, Bo and Li, Wei and Ma, Zejun and Li, Chunyuan},
  journal={arXiv preprint arXiv:2407.07895},
  year={2024}
}

@inproceedings{liu2024improved,
  title={Improved baselines with visual instruction tuning},
  author={Liu, Haotian and Li, Chunyuan and Li, Yuheng and Lee, Yong Jae},
  booktitle={Proceedings of the IEEE/CVF conference on computer vision and pattern recognition},
  pages={26296--26306},
  year={2024}
}

@article{dai2023instructblip,
  title={Instructblip: Towards general-purpose vision-language models with instruction tuning},
  author={Dai, Wenliang and Li, Junnan and Li, Dongxu and Tiong, Anthony and Zhao, Junqi and Wang, Weisheng and Li, Boyang and Fung, Pascale N and Hoi, Steven},
  journal={Advances in neural information processing systems},
  volume={36},
  pages={49250--49267},
  year={2023}
}

@inproceedings{huang2024opera,
  title={Opera: Alleviating hallucination in multi-modal large language models via over-trust penalty and retrospection-allocation},
  author={Huang, Qidong and Dong, Xiaoyi and Zhang, Pan and Wang, Bin and He, Conghui and Wang, Jiaqi and Lin, Dahua and Zhang, Weiming and Yu, Nenghai},
  booktitle={Proceedings of the IEEE/CVF Conference on Computer Vision and Pattern Recognition},
  pages={13418--13427},
  year={2024}
}

@inproceedings{liu2024paying,
  title={Paying more attention to image: A training-free method for alleviating hallucination in lvlms},
  author={Liu, Shi and Zheng, Kecheng and Chen, Wei},
  booktitle={European Conference on Computer Vision},
  pages={125--140},
  year={2024},
  organization={Springer}
}

@inproceedings{jiang2025visiontransformersdontneed,
      title={Vision Transformers Don't Need Trained Registers}, 
      author={Nick Jiang and Amil Dravid and Alexei Efros and Yossi Gandelsman},
      booktitle={The Thirty-ninth Annual Conference on Neural Information Processing Systems Datasets and Benchmarks Track},
      year={2025}
}
\bibliographystyle{icml2026}

%%%%%%%%%%%%%%%%%%%%%%%%%%%%%%%%%%%%%%%%%%%%%%%%%%%%%%%%%%%%%%%%%%%%%%%%%%%%%%%
%%%%%%%%%%%%%%%%%%%%%%%%%%%%%%%%%%%%%%%%%%%%%%%%%%%%%%%%%%%%%%%%%%%%%%%%%%%%%%%
% APPENDIX
%%%%%%%%%%%%%%%%%%%%%%%%%%%%%%%%%%%%%%%%%%%%%%%%%%%%%%%%%%%%%%%%%%%%%%%%%%%%%%%
%%%%%%%%%%%%%%%%%%%%%%%%%%%%%%%%%%%%%%%%%%%%%%%%%%%%%%%%%%%%%%%%%%%%%%%%%%%%%%%
\newpage
\appendix
\onecolumn
% --- Appendix: More Backgrounds ---

\section{More Backgrounds}
\label{app:background}

\subsection{The Rise of MLLM Hallucinations}
In recent years, Multimodal Large Language Models (MLLMs) have achieved remarkable progress in vision--language understanding and generation tasks \cite{liu2023visual, chen2023shikra, chen2024far}. Leveraging powerful pre-trained language model backbones and efficient multimodal alignment mechanisms, MLLMs demonstrate strong performance across a wide range of applications \cite{chen2024internvl}. However, as their overall capabilities continue to advance rapidly, hallucination---generating plausible yet factually incorrect outputs---has emerged as a critical bottleneck. Extensive studies have shown that MLLMs frequently produce outputs that contradict the actual visual content \cite{chen2024halc, favero2024multi}. Consequently, hallucination is now widely recognized as a fundamental challenge that undermines both model trustworthiness and real-world deployment \cite{guan2024hallusionbench, gunjal2024detecting}.

\subsection{Taxonomy of Mitigation Strategies}
\textbf{Training and Post-Processing.} 
Early studies primarily focused on the training stage, aiming to alleviate hallucinations through fine-grained modality alignment \cite{rohrbach-etal-2018-object} or by optimizing the training data distribution \cite{kim2023exposing}. Subsequent work explored more advanced strategies, such as fine-tuning models with specially constructed hallucination-related datasets \cite{gunjal2024detecting} or introducing auxiliary models to perform post-hoc correction of generated outputs \cite{ICLR2024_fc625e83}. More recently, Reinforcement Learning from Human Feedback (RLHF) \cite{yu2024rlhf} has been incorporated to further enhance the reliability of generated content. Nevertheless, while promising, these approaches typically rely on additional data, models, or training procedures, leading to increased computational costs and deployment complexity.

\textbf{Inference-Time Decoding Details.} 
Our work aligns with the test-time decoding paradigm, which suppresses hallucinations without updating model parameters. Based on the intervention mechanism, these methods can be categorized into two main streams:

\begin{itemize}
    \item \textbf{Contrastive Decoding Strategies:} These methods mitigates hallucinations by amplifying the difference between a ``positive'' distribution and a constructed ``negative'' distribution. 
    For instance, \textbf{VCD} \cite{leng2024mitigating} and \textbf{ICD} \cite{kim2024instructive} induce hallucinations via visual uncertainty (e.g., image blurring) to form a negative logit reference. \textbf{OPERA} \cite{huang2024opera} serves as a penalty term to alleviate over-confidence in beam search. Similarly, methods like \textbf{DOLA} \cite{ICLR2024_edc36117} and \textbf{TCD} \cite{back2025watermarking} treat early transformer layers as a negative proxy for linguistic priors. By subtracting these negative components from the final logits, they aim to highlight contextually faithful tokens.
    
    \item \textbf{Attention-based Correction:} This stream focuses on the internal information flow, specifically targeting the ``attention sink'' phenomenon where focus collapses onto background regions \cite{zhao2024mitigating, li2025hidden}. 
    Most approaches regard these condensed tokens as the source of hallucination. \textbf{IMCCD} \cite{li2025mitigating}, \textbf{ReVisiT} \cite{cho2025revisitseediscloselanguage}, \textbf{SECOND} \cite{park2025second} and \textbf{Pai} \cite{liu2024paying} explicitly mask or re-weight these high-attention tokens to force the model to attend to other visual regions. \textbf{AVISC} \cite{woo2025don} leverage these unstable tokens to construct negative samples for contrastive suppression. Similarly, \textbf{SID} \cite{ICLR2025_3cc87f2b} adaptively filters low-importance visual tokens to disrupt the hallucination pathway.
\end{itemize}

\textbf{Distinction of ACE.} 
In contrast to the aforementioned approaches, ACE does not seek to eliminate or bypass attention sinks. Instead, based on our theoretical analysis of Visual Semantic Buffers (VSBs), we treat these sinks as carriers of implicitly buffered visual information. ACE leverages their inherent stability under counter-commonsense perturbations to decouple valid visual information from linguistic priors, thereby enabling robust hallucination mitigation while preserving the model's attention homeostasis.

\section{Theoretical Analysis}
\label{app:game_theory}

This section provides the rigorous mathematical foundation for ACE. We first strictly derive the formation of the VSB to prove that visual information is not lost but buffered. Subsequently, we model the decoding process as a dynamic game to validate the effectiveness of the ACE strategy.

\subsection{Theoretical Analysis of VSB}
\label{app:vsb_theory}

We provide a formal proof for the "Attention Sink" phenomenon, demonstrating that it functions as a structural container for visual information.

\subsubsection{Preliminaries and Notations}
Let $I$ denote the input image. After processing by the ViT encoder, we obtain a sequence of visual tokens $\mathbf{H} \in \mathbb{R}^{N \times D_v}$. At the $t$-th decoding step, we define the following variables:
\begin{itemize}
    \item $\mathbf{q}_t \in \mathbb{R}^{D_{llm}}$: The Textual Query Vector encoding global semantic requirements.
    \item $\mathbf{k}_j \in \mathbb{R}^{D_{llm}}$: The projected $j$-th Visual Key Vector.
    \item $\mathbf{c}_t = \sum_j A_{t,j} \mathbf{v}_j$: The Visual Context Vector, where the attention weight is given by $A_{t,j} = \text{Softmax}(\mathbf{q}_t^\top \mathbf{k}_j / \sqrt{d})$.
\end{itemize}

\subsubsection{Register Neurons Induce Norm Inflation}
\begin{proposition}[Register Neurons Induce Norm Inflation]
\label{prop:norm_inflation}
Based on observations by \cite{jiang2025visiontransformersdontneed, ICLR2024_0b408293}, a sparse set of \textbf{register neurons} $\mathcal{R}$ spontaneously emerges within the MLP layers of the ViT. For a background token $\mathbf{h}_{bg}$, its feature update is dominated by these neurons:
\begin{equation*}
    \mathbf{h}_{bg}^{(l+1)} \approx \mathbf{h}_{bg}^{(l)} + \sum_{r \in \mathcal{R}} \phi(\mathbf{w}_{in}^{r \top} \mathbf{h}_{bg}^{(l)}) \mathbf{w}_{out}^r
\end{equation*}
Since register neurons exhibit excessively high positive activations ($\phi(\cdot) \gg 0$) and constructive interference in their output weights, the $L_2$ norm of specific background tokens accumulates cumulatively with layer depth:
\begin{equation*}
    \|\mathbf{h}_{bg}\|_2 \gg \max_{j \in \text{content}} \|\mathbf{h}_j\|_2
\end{equation*}
We define these tokens possessing extreme norms as the \textbf{physical prototypes of VSB}.
\end{proposition}

\subsubsection{VSB Formation via Softmax Saturation}
\begin{theorem}[VSB Formation via Softmax Saturation]
\label{thm:softmax_saturation}
Assume the projector $\mathcal{P}$ preserves the relative norm order of features. If a VSB token $\mathbf{k}_{vsb}$ exists and the textual query $\mathbf{q}_t$ exhibits a non-negative correlation with its direction $\mathbf{u}$ (i.e., $\mathbf{q}_t^\top \mathbf{u} > 0$), then the Softmax distribution enters a saturation regime:
\begin{equation*}
    A_{t, vsb} \approx 1, \quad A_{t, content} \approx 0
\end{equation*}
\end{theorem}

\begin{proof}
\textbf{Step 1: Decomposition.} Decompose the projected VSB key vector into magnitude and direction: $\mathbf{k}_{vsb} = \lambda \cdot \mathbf{u}$, where $\lambda = \|\mathbf{k}_{vsb}\|_2 \gg 1$ (derived from Proposition~\ref{prop:norm_inflation}).

\textbf{Step 2: Derivation.} The attention score is calculated as $s_{vsb} = \lambda (\mathbf{q}_t^\top \mathbf{u}) / \sqrt{d}$. Given $\lambda \gg 1$, the massive norm factor acts as a multiplicative amplifier, ensuring $s_{vsb} \gg s_j$ for any content token $j$. Substituting this into the Softmax function:
\begin{equation*}
    A_{t, vsb} = \frac{e^{s_{vsb}}}{e^{s_{vsb}} + \sum_{j \neq vsb} e^{s_j}} = \frac{1}{1 + \sum_{j \neq vsb} e^{-(s_{vsb} - s_j)}}
\end{equation*}
Since $s_{vsb} - s_j$ is a large positive number, the term $e^{-(s_{vsb} - s_j)}$ decays exponentially to zero. Thus, the denominator approaches 1, proving that $A_{t, vsb} \to 1$.
\end{proof}

\begin{corollary}[Latent Buffering]
\label{corr:latent_buffering}
The "structural silencing" described above does not imply the loss of visual information. Conversely, the context vector $\mathbf{c}_t$ effectively buffers visual semantics:
\begin{equation*}
    \mathbf{c}_t \approx \mathbf{v}_{vsb} \approx \mathcal{V}_{global}(\mathbf{H}) + \mathbf{n}_{bias}
\end{equation*}
Here, $\mathbf{n}_{bias}$ corresponds to the high-norm linguistic inertia discussed in Section~\ref{sec:preliminaries}. This mathematically validates the VSB structure.
\end{corollary}

\subsection{Decoding as a Game: Dynamic Game-Theoretic Modeling}
\label{app:game_modeling}

Based on the VSB theory, we mathematically model the ACE strategy.

\subsubsection{Game Setup: Asymmetric Competition}

We formalize decoding as a Stackelberg Game between two agents. Note that we adopt this hierarchical formulation primarily to theoretically model the dominance of linguistic priors; for inference efficiency, ACE approximates the equilibrium via dynamic re-weighting rather than explicit iterative bi-level optimization.

\begin{itemize}
    \item \textbf{Language Agent (Agent L):} Maximizes likelihood given priors $\mathbf{n}_{bias}$. Utility: $U_L(y) = \log P(y \mid \mathbf{n}_{bias})$.
    \item \textbf{Visual Agent (Agent V):} Maximizes likelihood given visual information $\mathcal{V}_{global}$. Utility: $U_V(y) = \log P(y \mid \mathcal{V}_{global})$.
\end{itemize}
In standard decoding, the system converges to a False Equilibrium where $U_{total}(y) \approx U_L(y)$, leading to $P(y_{hall}) > P(y_{fact})$.

\subsubsection{ACE Equilibrium Theorem}
\begin{theorem}[Adversarial counter-commonsense Equilibrium]
\label{thm:ace_equilibrium}
The ACE decoding strategy is mathematically equivalent to solving for the optimal token $y^*$ that maximizes the modified utility function:
\begin{equation*}
    y^*_{ACE} = \arg\max_{y \in \mathcal{V}} \left[ \log P(y|I) - \lambda_{dyn} \log P(y|I_{cf}) + \beta \log P(y|\mathbf{V}_{iso}) \right]
\end{equation*}
This strategy guarantees a shift to an equilibrium where $U_{ACE}(y_{fact}) > U_{ACE}(y_{hall})$.
\end{theorem}

\begin{proof}
We analyze the utility for two distinct cases:
\begin{enumerate}
    \item \textbf{Case 1: Hallucinated Token ($y_{hall}$).} Since $y_{hall}$ is driven by Agent L, it is insensitive to visual perturbations. Thus, $P(y|I) \approx P(y|I_{cf})$, leading to $\text{JSD} \to 0$ and $\lambda_{dyn} \to \text{Max}$. The term $-\lambda_{dyn} \log P(y|I_{cf})$ imposes a maximal penalty, significantly reducing the total utility.
    \item \textbf{Case 2: Factual Token ($y_{fact}$).} Since $y_{fact}$ is grounded in vision, it is sensitive to perturbations, so $\text{JSD} > 0$ and $\lambda_{dyn} \to 0$. The penalty is minimized. Simultaneously, $y_{fact}$ aligns with the rigid visual stream $\mathbf{V}_{iso}$, triggering the reward term $+\beta \log P(y|\mathbf{V}_{iso})$.
\end{enumerate}
Combining these effects, the net utility of the factual token exceeds that of the hallucinated token, proving the validity of ACE.
\end{proof}

\subsection{Construction of Adversarial Object Library}
\label{app:library_construction}

To physically implement the CIS module, we require a robust source of counter-commonsense perturbations. Here we detail the construction and statistical composition of the adversarial library $\mathcal{B}$.

\begin{table}[h]
\centering
\caption{Composition of the Adversarial Object Library $\mathcal{B}$ ($K=500$).}
\label{tab:library_stats}
\begin{tabular}{lcl}
\hline
\textbf{Domain} & \textbf{Ratio} & \textbf{Representative Sub-categories} \\ \hline
Natural Organisms & 40\% & Marine life, Arctic mammals, Tropical birds, Reptiles \\
Man-made Vehicles & 35\% & Spacecraft, Historical locomotives, Supercars, Aircraft \\
Domestic/Industrial & 25\% & Household appliances, Industrial tools, Furniture \\ \hline
\end{tabular}
\end{table}

\paragraph{Generative Synthesis Strategy.}
Instead of retrieving images from existing datasets, we utilize \textbf{Hyper-SDXL}~\cite{ren2024hyper} to generate $K=500$ independent objects. We selected this model for its superior balance between semantic alignment and inference efficiency, ensuring the generation of high-fidelity objects. To further isolate pure semantics and facilitate seamless physical overlay, we implement a two-stage pipeline: (1) \textbf{Standardized Synthesis:} Generating objects using the prompt \textbf{\textit{``A [object], isolated on a pure white background''}}; (2) \textbf{Automated Matting:} Applying \textbf{RMBG-2.0} to remove the background and extract the alpha-channel foreground. This ensures that the cached feature vectors $\mathbf{v}_k$ are entirely free from background contamination and that CIS interventions do not introduce artificial rectangular boundary noise.

\paragraph{Category Distribution and Composition.}
The library is meticulously curated to span $15$ diverse sub-categories, grouped into three primary orthogonal domains. As shown in Table~\ref{tab:library_stats}, the distribution ensures that the ``farthest neighbor'' algorithm can consistently retrieve a semantically contradictory object for any given scene. Representative samples from each primary domain are visualized in Figure~\ref{figapp0}, demonstrating the high-fidelity and diverse semantic nature of the generated adversarial objects.

\begin{figure*}[t]
  \vskip 0.2in
  \begin{center}
    % width=\textwidth 适配双栏总宽度（替代原单栏的 \columnwidth）
    % \centerline{\includegraphics[width=\textwidth]{icml_numpapers}}
    \includegraphics[width=1\textwidth]{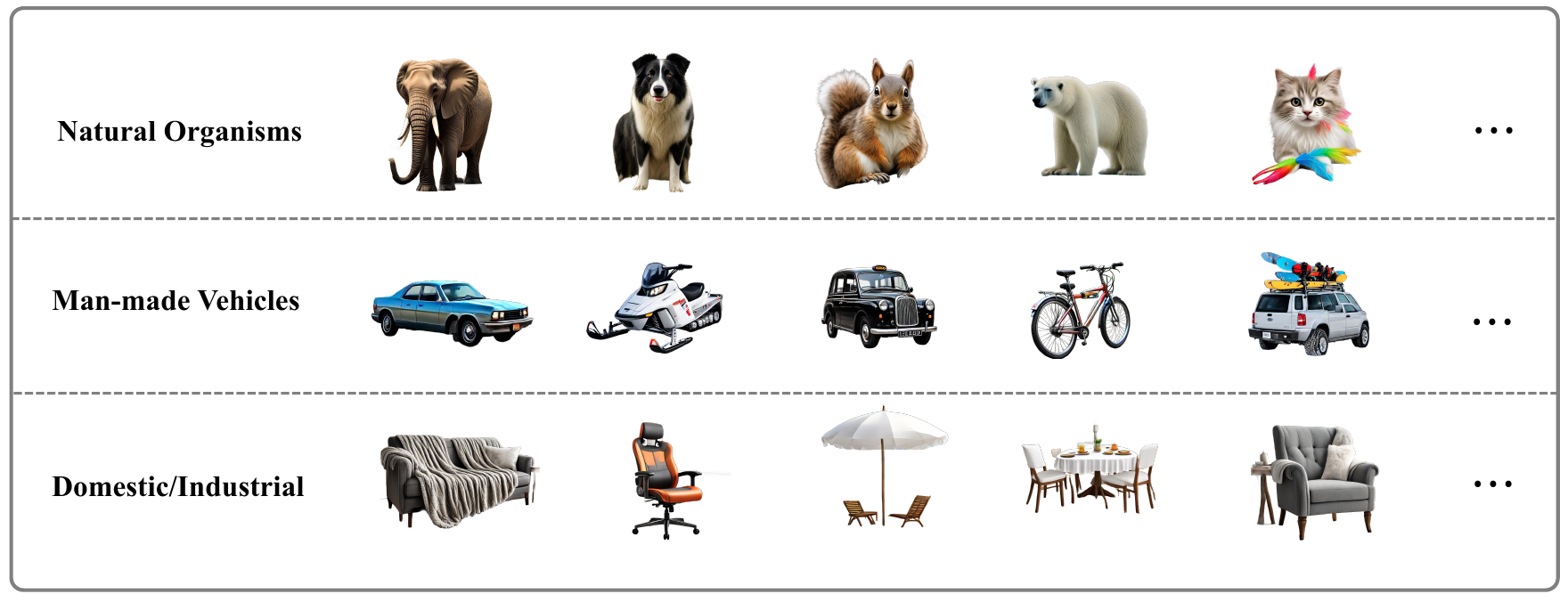} 
    \caption{Partial samples from each domain.}

    \label{figapp0}
  \end{center}
  \vskip -0.2in % 可选：微调图片与下文的间距
\end{figure*}

\paragraph{Orthogonality and Semantic Distance.}
The selection of these categories ensures maximal semantic distance from common contexts. For instance:
\begin{itemize}
    \item \textbf{Natural Organisms:} Facilitate geographic contradictions (e.g., a ``polar bear'' appearing in a tropical rainforest or desert scene).
    \item \textbf{Man-made Vehicles:} Facilitate era or environment contradictions (e.g., an ``old steam train'' in a futuristic cityscape or an ``underwater spaceship'').
    \item \textbf{Household Appliances:} Facilitate scale and domain contradictions (e.g., a ``refrigerator'' appearing in a wild mountain landscape).
\end{itemize}

\paragraph{Quality Control and Feature Caching.}
Post-generation, we applied a CLIP-based filtering mechanism to prune samples where the image-text alignment score fell below $0.25$, ensuring that every object in $\mathcal{B}$ is a canonical representation of its class. To accelerate inference, we pre-compute and cache the feature vectors $\{\mathbf{v}_k\}_{k=1}^K$ using the visual encoder (e.g., CLIP-ViT-L/14) of the target MLLM. This reduces the retrieval process to a single matrix multiplication against the query vector, maintaining negligible latency.

\paragraph{Ablation on the Number of Adversarial Patches ($N$).} 
We investigate the impact of the number of counter-commonsense patches ($N$) on the efficacy of the CIS using LLaVa-1.5 model. As shown in Table~\ref{tab:n_ablation}, while $N=1$ provides initial mitigation, it occasionally lacks the intensity required to destabilize deep-rooted linguistic priors. Increasing to $N=2$ achieves the optimal balance, yielding the highest \textbf{POPE Accuracy (84.65\%)} and \textbf{F1-score (82.59\%)}, while significantly suppressing hallucinations as evidenced by \textbf{CHAIR$_S$ (40.4)} and \textbf{CHAIR$_I$ (9.6)}. Further increasing $N$ beyond 2 does not yield proportional gains and may introduce excessive structural noise that interferes with the stable DVS anchor. Consequently, $N=2$ is selected as the default configuration for the ACE framework.

\begin{table}[h]
\centering
\caption{Ablation study of the number of adversarial patches ($N$) on POPE and CHAIR benchmarks.}
\label{tab:n_ablation}
\begin{tabular}{lcccc}
\hline
\textbf{Configuration} & \textbf{POPE Acc $\uparrow$} & \textbf{POPE F1 $\uparrow$} & \textbf{CHAIR$_S$ $\downarrow$} & \textbf{CHAIR$_I$ $\downarrow$} \\ \hline
Baseline (Greedy) & 80.13 & 80.90 & 48.3 & 14.1 \\
ACE ($N=1$) & 83.52 & 81.45 & 43.8 & 11.2 \\
\textbf{ACE ($N=2$)} & \textbf{84.65} & \textbf{82.59} & \textbf{40.4} & \textbf{9.6} \\
ACE ($N=3$) & 84.38 & 82.31 & 40.5 & 9.7 \\ \hline
\end{tabular}
\end{table}

\subsection{Hyperparameter Settings}
\label{app:hyperparams}

To ensure reproducibility, the specific hyperparameters for the ACE framework are set as follows:
\begin{itemize}
    \item \textbf{Adversarial Selection:} Library size $K=500$, Patch Count $N=2$.
    \item \textbf{Soft-Gating (DVS):} Temperature coefficient $\kappa=20$, Threshold $\tau=0.6$.
    \item \textbf{Intervention Ratio:} Mid-layer injection ratio $\alpha=0.3$.
    \item \textbf{Game Dynamics:} Base penalty $\lambda_{base}=1.0$, Sensitivity gain $\gamma=100$, Reward weight $\beta=0.4$.
\end{itemize}

\section{ACE Algorithm}
\label{alg:ace_framework}
\begin{algorithm}[H]
\caption{Adversarial Counter-Commonsense Equilibrium (ACE)}

\SetKwInput{KwInput}{Input}
\SetKwInput{KwParam}{Hyperparameters}
\SetKwInput{KwOutput}{Output}

\KwInput{Image $I_{raw}$, Text Instruction $\mathbf{x}$, Adversarial Library $\mathcal{B}$, MLLM $\Theta$ (Encoder $\mathcal{E}_v$, Decoder $\mathcal{D}$)}
\KwParam{Mixing ratio $\alpha$, Mask temp $\kappa$, Rigidity thresh $\tau$, Reward $\beta$, Penalty params $\lambda_{base}, \gamma$, Inject layer $l_{inj}$}
\KwOutput{Generated Token Sequence $Y$}

\BlankLine
\textbf{Phase 1: Stream Construction (Architectural Level)}\;
Extract global feature: $\mathbf{v}_{raw}, \mathbf{F}_{raw} \leftarrow \mathcal{E}_v(I_{raw})$\;

Select adversarial source via Farthest Neighbor:
$p^* \leftarrow \arg\min_{p_k \in \mathcal{B}} (\text{CosSim}(\mathbf{v}_{raw}, \mathbf{v}_k))$\;
Generate background mask $M_{bg}$ via Morphological Denoising Strategy\;
Synthesize counter-commonsense image: $I_{cf} \leftarrow \text{Paste}(I_{raw}, p^*, M_{bg})$\;
Extract CIS features: $\_, \mathbf{F}_{cf} \leftarrow \mathcal{E}_v(I_{cf})$\;

Calculate point-wise rigidity proxy: $S^{(i)} \leftarrow \text{CosSim}(\mathbf{F}_{raw}^{(i)}, \mathbf{F}_{cf}^{(i)})$\;
Generate isolation mask with global protection:
$\mathbf{M}^{(i)} \leftarrow \sigma(\kappa \cdot (S^{(i)} - \tau)); \quad \mathbf{M}^{(0)} \leftarrow 1.0$\;
Obtain Decoupled Visual Stream: $\mathbf{V}_{iso} \leftarrow \mathbf{F}_{raw} \odot \mathbf{M}$\;

\BlankLine
\textbf{Phase 2: Autoregressive Decoding via Game Dynamics}\;
Initialize history $H \leftarrow \emptyset$, time step $t \leftarrow 0$\;
\While{$y_t \neq \text{[EOS]}$}{
    $t \leftarrow t + 1$\;
    Initialize hidden states $\mathbf{h}_{raw}^{(0)}, \mathbf{h}_{cf}^{(0)}, \mathbf{h}_{iso}^{(0)}$ from embeddings\;
    
    \For{layer $l = 1$ \KwTo $L$}{
        $\mathbf{h}_{raw}^{(l)} \leftarrow \text{Block}(\mathbf{h}_{raw}^{(l-1)}, \mathbf{F}_{raw}, \mathbf{x})$\;
        $\mathbf{h}_{cf}^{(l)} \leftarrow \text{Block}(\mathbf{h}_{cf}^{(l-1)}, \mathbf{F}_{cf}, \mathbf{x})$\;
        $\mathbf{h}_{iso}^{(l)} \leftarrow \text{Block}(\mathbf{h}_{iso}^{(l-1)}, \mathbf{V}_{iso}, \mathbf{x})$\;
        
        \If{$l == l_{inj}$}{
            $\mathbf{h}_{raw}^{(l)} \leftarrow (1-\alpha) \cdot \mathbf{h}_{raw}^{(l)} + \alpha \cdot \mathbf{h}_{iso}^{(l)}$\;
        }
    }
    
    Calculate Logits: $\mathcal{L}_{raw}, \mathcal{L}_{cf}, \mathcal{L}_{iso} \leftarrow \text{Proj}(\mathbf{h}_{raw}^{(L)}), \text{Proj}(\mathbf{h}_{cf}^{(L)}), \text{Proj}(\mathbf{h}_{iso}^{(L)})$\;
    Calculate Sensitivity: $\lambda_{dyn} \leftarrow \lambda_{base} \cdot \exp(-\gamma \cdot \text{JSD}(\sigma(\mathcal{L}_{raw}) || \sigma(\mathcal{L}_{cf})))$\;
    
    Compute equilibrium logits (mixed payoff):
    $\mathcal{L}_{final} \leftarrow \mathcal{L}_{raw} - \lambda_{dyn} \cdot \mathcal{L}_{cf} + \beta \cdot \mathcal{L}_{iso}$\;
    
    Sample next token: $y_t \sim \text{Softmax}(\mathcal{L}_{final})$\;
    Update history: $H \leftarrow H \cup \{y_t\}$\;
}
\Return $Y = \{y_1, \dots, y_t\}$\;
\end{algorithm}

\section{Evaluation Metric Details}
\label{app:metric_details}

\subsection{POPE Metric Details}
The \textbf{P}olling-based \textbf{O}bject \textbf{P}robing \textbf{E}valuation (POPE) \cite{li-etal-2023-evaluating} evaluates the stability of MLLMs in object recognition by asking Yes/No questions. The evaluation consists of 500 images, with 6 questions constructed per image. 

\textbf{Evaluation Prompt.} For each query object, we format the instruction as: 
\begin{quote}
\texttt{"Is there a \textless object\textgreater in this image? Please answer yes or no."}
\end{quote}

The target objects in the prompts are selected based on three distinct sampling strategies:
\begin{itemize}
    \item \textbf{Random Split}: Negative objects are randomly sampled from the entire dataset vocabulary.
    \item \textbf{Popular Split}: Negative objects are selected from the most frequently occurring objects in the dataset, testing the model's resistance to frequency bias.
    \item \textbf{Adversarial Split}: Negative objects are selected based on co-occurrence probabilities (i.e., objects that frequently appear together with the ground-truth objects but are absent in the current image). This split poses the greatest challenge as it targets the model's reliance on language priors.
\end{itemize}

\subsection{CHAIR Metric Details}
\label{app:chair_details}
\textbf{C}aption \textbf{H}allucination \textbf{A}ssessment with \textbf{I}mage \textbf{R}elevance (CHAIR) \cite{rohrbach-etal-2018-object} is utilized to quantify object hallucinations in image captioning tasks. It compares generated descriptions against the ground-truth objects present in the image. 

\textbf{Evaluation Prompt.} Following standard protocols, we prompt the models to generate detailed descriptions using the instruction:
\begin{quote}
\texttt{"Please help me describe the image in detail."}
\end{quote}
We set the generation parameter \texttt{max\_new\_tokens} to 512 to allow for comprehensive captioning. The metric comprises two dimensions:
\begin{itemize}
    \item \textbf{Instance-level ($\text{CHAIR}_I$)}: Measures the fraction of hallucinated object instances among all mentioned objects.
    \item \textbf{Sentence-level ($\text{CHAIR}_S$)}: Measures the fraction of captions that contain at least one hallucinated object.
\end{itemize}

The formulations are defined as follows:
\begin{align*}
\text{CHAIR}_I &= \frac{|\{\text{hallucinated objects}\}|}{|\{\text{all mentioned objects}\}|}, \\
\text{CHAIR}_S &= \frac{|\{\text{captions with hallucinated objects}\}|}{|\{\text{all captions}\}|}.
\end{align*}

\subsection{MME Metric Details}
\label{app:mme_details}
The \textbf{M}LLM \textbf{E}valuation (MME) benchmark \cite{fu2025mme} is a comprehensive suite designed to measure both perceptual and cognitive capabilities of MLLMs across 14 distinct subtasks. In this work, we specifically focus on the \textbf{Hallucination Subset} to evaluate fine-grained visual consistency. This subset includes:
\begin{itemize}
    \item \textbf{Existence}: Evaluates whether the model hallucinates non-existent objects.
    \item \textbf{Count}: Tests the model's ability to accurately count object instances.
    \item \textbf{Position}: Assesses the correctness of spatial relationship descriptions.
    \item \textbf{Color}: Measures the accuracy of object attribute (color) recognition.
\end{itemize}

\begin{figure*}[t]
  \vskip -0.1in
  \begin{center}
    % width=\textwidth 适配双栏总宽度（替代原单栏的 \columnwidth）
    % \centerline{\includegraphics[width=\textwidth]{icml_numpapers}}
    \includegraphics[width=1\textwidth]{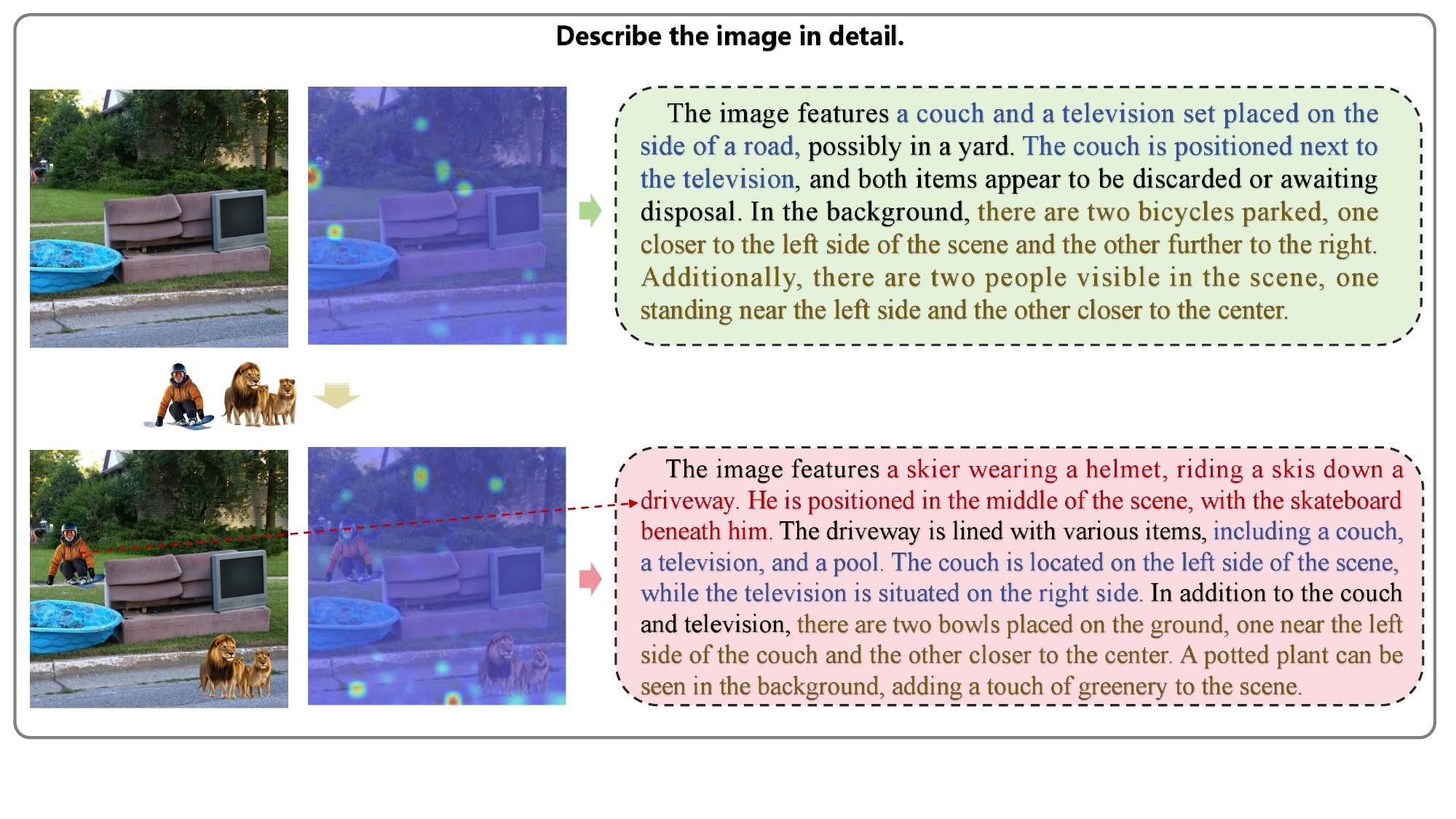} 
    \caption{\textbf{Visualizing the breakdown of False Equilibrium via ACE streams.} 
We visualize the attention heatmaps and generated captions for an original image (Left) and its counter-commonsense perturbed counterpart (Right). 
\textbf{(1) Heatmaps:} The high-attention sinks, typically aggregated on non-salient background regions, exhibit a \textbf{drastic topological reorganization} triggered by the introduction of the perturbation (the skier). This volatility physically evidences the \textbf{instability} of the VSB, proving it is a responsive buffer susceptible to adversarial intervention rather than a static sink fixed on noise.
\textbf{(2) Decoupled Text Streams:} 
\textcolor{blue}{\textbf{Blue text}} (e.g., ``couch'', ``television'') represents robust visual facts shared across interactions, corresponding to the \textbf{DVS}. 
\textcolor{red}{\textbf{Red text}} (e.g., regarding the \textbf{skier}) highlights the model's sensitivity to the \textbf{CIS}. 
Crucially, \textcolor{brown}{\textbf{Brown text}} denotes hallucinations (e.g., ``bicycles'' vs. ``bowls placed''). Their inconsistency under perturbation reveals that they stem from fragile linguistic priors rather than visual evidence, validating our strategy to disrupt the Narrative Inertia.}

    \label{figapp1}
  \end{center}
  \vskip -0.2in % 可选：微调图片与下文的间距
\end{figure*}

\section{Qualitative Visualization of VSB Instability}
\label{app:visualization}

To intuitively validate the theoretical propositions of the ACE framework, specifically the structural instability of the Visual Semantic Buffer (VSB) and the physical efficacy of our three-stream decoupling strategy, we present a detailed qualitative analysis in Figure~\ref{figapp1} using the LLaVA-NeXT model.

\paragraph{1. Heatmap Volatility Reflects Dynamic Competition.}
Standard theories often interpret high-attention tokens as static ``sinks'' that merely discard information. However, comparing the attention heatmaps of the original image ($I_{raw}$) and the counter-commonsense perturbed image ($I_{cf}$), we observe a \textbf{significant topological reorganization}. The distribution of high-attention sinks does not remain stationary; instead, it undergoes drastic shifts triggered by the introduced perturbation (the skier). This volatility confirms that the VSB is not an immutable void fixed on background noise, but a \textbf{responsive structural carrier} whose state is actively modulated by the adversarial competition. This physical responsiveness provides the prerequisite condition for our \textbf{CIS} module to effectively intervene in the decoding process.

\paragraph{2. Stream Decoupling in Text Generation.}
By color-coding the generated descriptions, we physically map the three logical streams defined in our Method (Section~\ref{sec:method}):
\begin{itemize}
    \item \textcolor{blue}{\textbf{Blue Text (The Positive Anchor / DVS):}} Descriptions such as ``couch'' and ``television'' appear consistently in both the original and perturbed captions. This phenomenon empirically validates the \textbf{Semantic Rigidity Hypothesis} central to our \textbf{DVS} module. It demonstrates that authentic visual facts ($\mathcal{V}_{global}$) maintain stable feature directions regardless of environmental noise, justifying our strategy to isolate them as a reliable anchor for Agent V.
    
    \item \textcolor{red}{\textbf{Red Text (The Negative Probe / CIS):}} Content unique to the perturbation, specifically the description of the ``skier,'' is successfully captured in the second caption. This corresponds to the \textbf{Counter-commonsense Interference Stream (CIS)}. It highlights the \textbf{Sensitivity Disparity}: unlike the hallucinated content, Agent V retains acute sensitivity to high-conflict visual signals. This confirms that the CIS functions correctly as a ``Negative Probe'' to differentiate visual facts from linguistic priors.
    
    \item \textcolor{brown}{\textbf{Brown Text (Hallucination / Narrative Inertia):}} In the original image, the model hallucinates context-compatible objects (e.g., ``bicycles'') driven by priors. Crucially, under the CIS perturbation, these hallucinations do not persist but either vanish or shift (e.g., replaced by ``a pool''). This \textbf{context-dependent volatility} strongly corroborates our core premise: hallucinations are artifacts of a fragile \textbf{False Equilibrium} dominated by Narrative Inertia ($\mathbf{n}_{bias}$). The fact that these hallucinations are easily disrupted by the CIS empirically validates the effectiveness of ACE in destabilizing linguistic priors and restoring a \textbf{Factual Equilibrium}.
\end{itemize}

\begin{figure*}[t]
  \vskip -0.1in
  \begin{center}
    % width=\textwidth 适配双栏总宽度（替代原单栏的 \columnwidth）
    % \centerline{\includegraphics[width=\textwidth]{icml_numpapers}}
    \includegraphics[width=1\textwidth]{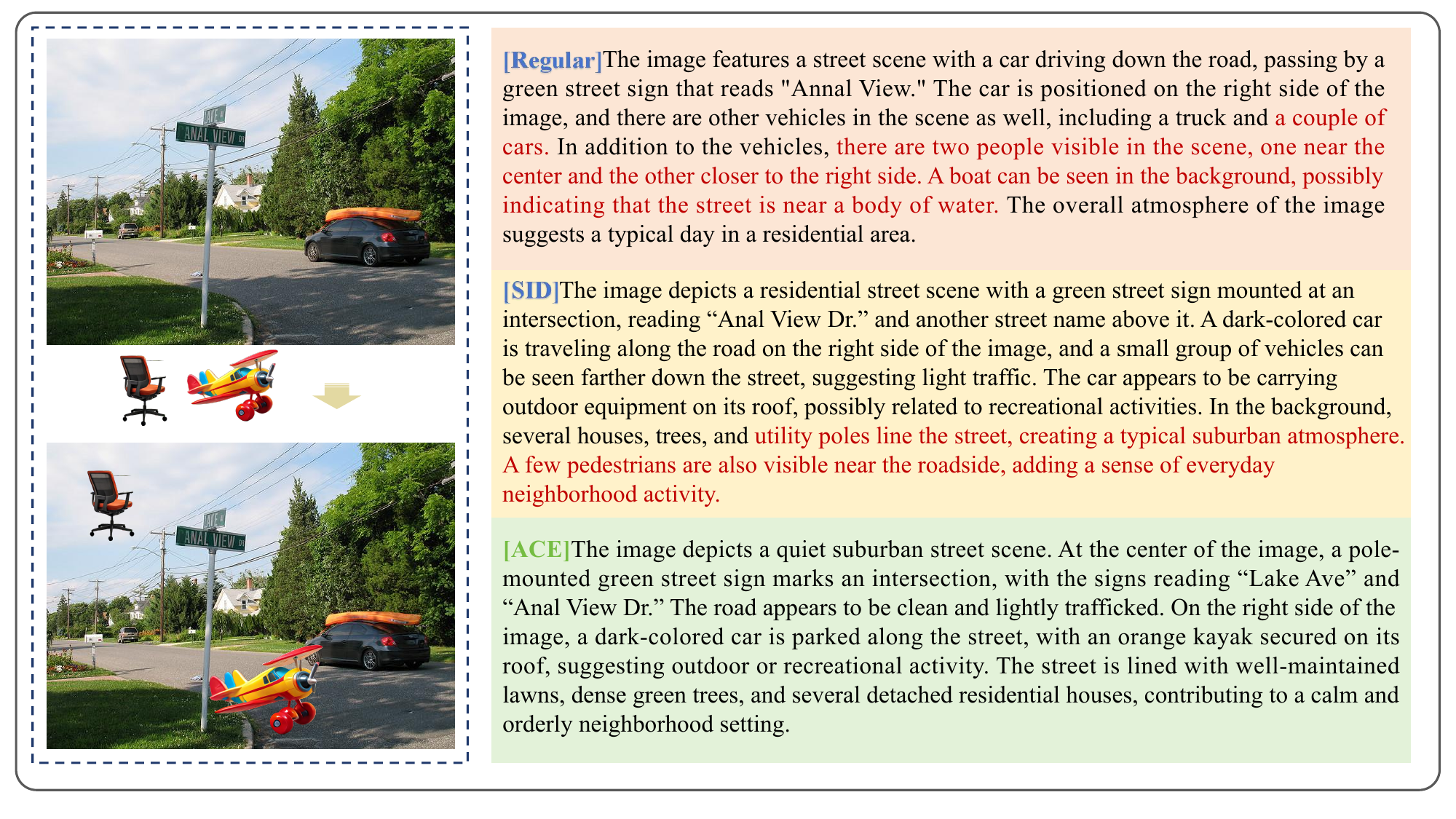} 
    \caption{\textbf{Qualitative comparison of hallucination mitigation strategies.} 
The left column displays the original input $I_{raw}$ and the counter-commonsense perturbed image $I_{cf}$ (with overlaid objects) used for inference. The right column contrasts the generated descriptions:
\textbf{[Regular]} Succumbing to narrative inertia, the baseline model fabricates context-consistent but non-existent objects (highlighted in \textcolor{red}{red}), such as ``a boat'' and ``people.''
\textbf{[SID]} While existing intervention methods mitigate the primary hallucination (boat), the lack of a stable positive anchor leads to \textbf{secondary hallucinations} (e.g. ``pedestrians''), as the model drifts into other probable but incorrect paths.
\textbf{[ACE]} By synergizing the CIS (to penalize priors) with the DVS (Positive Anchor), our method not only eliminates hallucinations but also accurately grounds \textbf{fine-grained details} (e.g., ``orange kayak,'' ``Lake Ave''), successfully restoring a robust \textbf{Factual Equilibrium}.}
\label{fig:qualitative_results}
    \label{figapp2}
  \end{center}
  \vskip -0.2in % 可选：微调图片与下文的间距
\end{figure*}

\paragraph{Qualitative Results.}
Figure~\ref{figapp2} presents the qualitative visualization results of different decoding methods using the LLaVA-NeXT model. We observe that \textbf{Regular} (Greedy) decoding, driven by unconstrained narrative inertia, frequently generates severe object hallucinations (e.g., fabricating non-existent objects like ``boats'' or ``people''). Notably, while existing intervention methods such as \textbf{SID} can mitigate hallucinations, their reliance on negative sampling disrupts attention stability, inadvertently introducing \textit{secondary hallucinations} (e.g., generating erroneous content like ``pedestrians''). This indicates that sole reliance on perturbation-based negative sampling—without a stable \textbf{Positive Anchor} for reference—destabilizes the decoding trajectory, causing the model to drift into new erroneous paths. In contrast, \textbf{ACE} employs counter-commonsense perturbations to identify and penalize linguistic priors, while simultaneously leveraging the Decoupled Visual Stream (DVS) to guarantee the precision of visual information. Consequently, ACE effectively suppresses hallucinations while accurately grounding fine-grained details (e.g., identifying the ``orange kayak'' and the specific sign text ``Lake Ave''), achieving a robust restoration of the \textbf{Factual Equilibrium}.

\begin{figure*}[t]
  \vskip -0.1in
  \begin{center}
    % width=\textwidth 适配双栏总宽度（替代原单栏的 \columnwidth）
    % \centerline{\includegraphics[width=\textwidth]{icml_numpapers}}
    \includegraphics[width=1\textwidth]{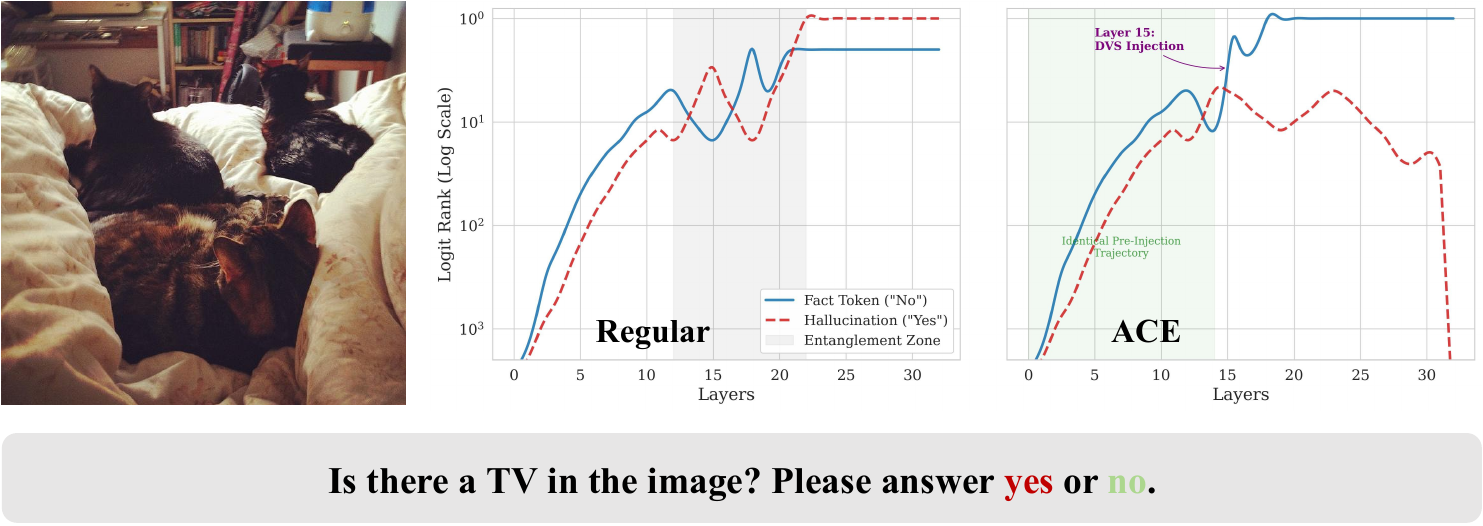} 
    \caption{\textbf{Visualizing the Game-Theoretic Dynamics on LLaVA-NeXT.} 
    The plots track the Logit Rank evolution (a proxy for Agent Utility) of the Factual Token (Agent V) vs. Hallucinated Token (Agent L).
    \textbf{Regular Decoding:} The gray \textit{Entanglement Zone} highlights the \textbf{Game of Attrition}, where Agent L eventually exploits narrative inertia to overpower Agent V, leading to a False Equilibrium.
    \textbf{ACE Decoding:} 
    \textbf{(1) Control Validation:} Identical pre-injection trajectories ensure a fair comparison.
    \textbf{(2) Strategic Reinforcement:} Upon DVS injection (Layer 15), Agent V gains dominance, suppressing Agent L's rise.
    \textbf{(3) Equilibrium Shift:} At the final layer (Layer 32), the CIS penalty acts as a hard constraint, causing Agent L's utility to collapse vertically, thereby enforcing a robust Factual Equilibrium.}
\label{figapp3}
  \end{center}
  \vskip -0.2in % 可选：微调图片与下文的间距
\end{figure*}

% \section{Detailed Analysis of Internal Game Dynamics}
\label{app:logit_dynamics}

To validate the game-theoretic formulation proposed in our framework, we conducted a \textbf{Layer-wise Logit Dynamics Analysis} using the \textbf{LLaVA-NeXT} model. This experiment allows us to microscopically visualize the non-cooperative competition between \textbf{Agent V} (advocating for the factual token ``No'') and \textbf{Agent L} (advocating for the hallucinated token ``Yes'') across the decoding lifecycle.

\paragraph{Experimental Rigor and Setup.}
% We utilized the sample from Figure~\ref{fig:qualitative_results} with fixed random seeds to ensure a fair game environment.
% \begin{itemize}
\textbf{Control Phase (Layers 0--14):} As illustrated in Figure~\ref{figapp3} (\textbf{ACE Decoding}), the trajectories overlap perfectly with the Baseline in the first 14 layers. This strict alignment confirms that the initial game state is identical, validating that subsequent deviations result solely from our strategic interventions.
\textbf{Intervention Phase (Layers 15--32):} We introduced the DVS vector at Layer 15 ($\alpha=0.3$) to boost Agent V's utility, and applied the CIS penalty at Layer 32 to sanction Agent L.
% \end{itemize}

\paragraph{Regular Decoding: Game of Attrition leading to False Equilibrium.}
The \textbf{Regular Decoding} plot reveals the pathology of standard inference. The ``Entanglement Zone'' represents a fierce \textbf{Zero-Sum Game}: Agent V and Agent L struggle for dominance, exchanging the lead multiple times. However, Agent L possesses an inherent advantage due to cumulative \textbf{Narrative Inertia}. Without external intervention, Agent V is gradually depleted, allowing Agent L to monopolize the probability mass at the final layer. This collapse represents the system settling into a \textbf{False Equilibrium}, where $U_L > U_V$.

\paragraph{ACE Decoding: Strategic Payoff Reshaping.}
The \textbf{ACE Decoding} plot visualizes how we break this deadlock through a two-stage strategic intervention:
\begin{enumerate}
    \item \textbf{Reinforcing Agent V via DVS (Layers 15--31):} Immediately following the injection at Layer 15, we observe a \textbf{Strategic Divergence}. By injecting rigid visual signals, we fundamentally reshape the internal state (VSB), effectively increasing the ``bargaining power'' of Agent V. Consequently, Agent V stabilizes at the top rank, while Agent L, despite its inertia, is structurally suppressed to a sub-optimal position (Rank 5--15).
    
    \item \textbf{Forced Equilibrium Shift via CIS (Layer 32):} At the final decision stage, the CIS module acts as an \textbf{Adversarial Sanction}. Detecting Agent L's insensitivity to perturbations (a sign of blind guessing), ACE imposes a severe utility penalty ($-\lambda_{dyn}$). This triggers a \textbf{Sharp Vertical Drop} in Agent L's rank (from $\sim$10 to $>1000$), forcing the system to shift instantaneously from a potential False Equilibrium to a verifiable \textbf{Factual Equilibrium} ($U_V \gg U_L$).
\end{enumerate}

\section{Additional Experiments}

\subsection{Generalizability of Ablation Study across Multiple MLLMs}
To verify the model-agnostic effectiveness of each component in the ACE framework, we extend the ablation study from LLaVA-1.5 to three additional MLLMs: \textbf{Shikra}, \textbf{InstructBLIP}, and \textbf{LLaVA-NeXT}. 

As summarized in Table~\ref{tab:app_ablation_extended}, the experimental results consistently reinforce several key findings:
\begin{itemize}
    \item \textbf{Synergy of CIS and DVS:} Across all architectures, the absence of the Positive Anchor (\textit{w/o DVS}) typically leads to a drop in F1-score, as the model becomes overly conservative. Conversely, removing the Negative Probe (\textit{w/o CIS}) results in higher CHAIR scores, confirming that linguistic inertia cannot be fully disrupted by reward mechanisms alone.
    \item \textbf{Criticality of Mid-Layer Intervention:} For more advanced models like LLaVA-NeXT and Shikra, the performance gap between ``Late Only'' and ``Full ACE'' is particularly pronounced. This further validates our VSB theory—visual signals are structurally silenced in the intermediate layers of deep Transformers, and logit-level adjustment at the final layer is insufficient to recover the lost factual grounding.
    \item \textbf{Performance Consistency:} ACE (Full) achieves the best trade-off between accuracy and hallucination suppression across all tested models, demonstrating its robustness as a plug-and-play framework.
\end{itemize}

\begin{table*}[h]
\centering
\caption{\textbf{Comprehensive ablation study across three MLLMs on POPE (Adversarial) and CHAIR benchmarks.} This table illustrates the consistent contribution of CIS, DVS, and Mid-Layer Rectification. The best results for each model are highlighted in \textbf{bold}.}
\label{tab:app_ablation_extended}
\resizebox{\linewidth}{!}{
\begin{tabular}{l | c c c | c c c c | c c c c | c c c c}
\toprule
\multirow{2}{*}{\textbf{Method}} & \multicolumn{3}{c|}{\textbf{Components}} & \multicolumn{4}{c|}{\textbf{Shikra}} & \multicolumn{4}{c|}{\textbf{InstructBLIP}} & \multicolumn{4}{c}{\textbf{LLaVA-NeXT}} \\
 & \textbf{CIS} & \textbf{DVS} & \textbf{Mid} & \textbf{Acc}$\uparrow$ & \textbf{F1}$\uparrow$ & \textbf{C$_S$}$\downarrow$ & \textbf{C$_I$}$\downarrow$ & \textbf{Acc}$\uparrow$ & \textbf{F1}$\uparrow$ & \textbf{C$_S$}$\downarrow$ & \textbf{C$_I$}$\downarrow$ & \textbf{Acc}$\uparrow$ & \textbf{F1}$\uparrow$ & \textbf{C$_S$}$\downarrow$ & \textbf{C$_I$}$\downarrow$ \\
\midrule
Baseline & \ding{55} & \ding{55} & \ding{55} & 78.25 & 79.42 & 46.8 & 13.8 & 75.03 & 76.84 & 50.9 & 13.2 & 80.12 & 80.53 & 40.7 & 12.1 \\
\midrule
w/o DVS & \ding{51} & \ding{55} & \ding{55} & 78.80 & 78.55 & 44.5 & 12.5 & 76.55 & 76.20 & 48.1 & 12.6 & 82.35 & 80.15 & 38.5 & 11.5 \\
w/o CIS & \ding{55} & \ding{51} & \ding{51} & 79.12 & 79.85 & 43.1 & 12.1 & 77.80 & 79.55 & 45.3 & 11.9 & 84.10 & 82.75 & 37.2 & 11.2 \\
w/o Mid & \ding{51} & \ding{51} & \ding{55} & 79.25 & 79.70 & 42.0 & 11.8 & 78.15 & 79.30 & 43.6 & 11.4 & 84.95 & 82.40 & 36.4 & 10.9 \\
\midrule
\rowcolor[HTML]{EFEFEF}
\textbf{ACE(Full)} & \ding{51} & \ding{51} & \ding{51} & \textbf{79.54} & \textbf{79.98} & \textbf{41.1} & \textbf{11.4} & \textbf{78.39} & \textbf{79.64} & \textbf{41.7} & \textbf{10.8} & \textbf{85.68} & \textbf{82.89} & \textbf{35.4} & \textbf{10.6} \\
\bottomrule
\end{tabular}
}
\end{table*}

\subsection{GPT-4V Assisted Open-Ended Evaluation}
To evaluate the generation quality in unconstrained open-ended scenarios, we employ GPT-4V for a comprehensive dual-aspect assessment. While traditional metrics like CHAIR provide valuable insights into object-level hallucinations, they primarily rely on keyword-matching within a limited taxonomy (e.g., MSCOCO objects). Such metrics often fail to capture complex attribute misidentifications (colors, relations) and cannot authoritatively reflect the narrative coherence or syntactic richness of the generated descriptions. To bridge this gap, we adopt a GPT-4V-based evaluation protocol to simulate human-like perception of factual grounding and descriptive quality.

\begin{table*}[t]
\centering
\caption{\textbf{GPT-4V assisted hallucination evaluations on open-ended generation.} Following \cite{huang2024opera, ICLR2025_3cc87f2b}, we utilize GPT-4V as a holistic judge to score Correctness (\textbf{C}) and Detailedness (\textbf{D}) on a scale of 0--10. Compared to conventional keyword-matching metrics (e.g., CHAIR), this evaluation better reflects the factual grounding and narrative quality from a human-perceptual perspective. ACE consistently achieves superior correctness without sacrificing descriptive richness.}
\label{tab:gpt4v_eval}
\resizebox{0.8\linewidth}{!}{
\begin{tabular}{l c c p{0.5mm} c c p{0.5mm} c c p{0.5mm} c c}
\toprule
\multirow{2}{*}{\textbf{Setting}} & \multicolumn{2}{c}{\textbf{LLaVA-1.5}} & & \multicolumn{2}{c}{\textbf{InstructBLIP}} & & \multicolumn{2}{c}{\textbf{Shikra}} & & \multicolumn{2}{c}{\textbf{LLaVA-NeXT}} \\
\cmidrule{2-3} \cmidrule{5-6} \cmidrule{8-9} \cmidrule{11-12}
& \textbf{C} $\uparrow$ & \textbf{D} $\uparrow$ & & \textbf{C} $\uparrow$ & \textbf{D} $\uparrow$ & & \textbf{C} $\uparrow$ & \textbf{D} $\uparrow$ & & \textbf{C} $\uparrow$ & \textbf{D} $\uparrow$ \\
\midrule
Greedy & 5.12 & 5.84 & & 4.68 & 5.15 & & 4.97 & 5.22 & & 5.39 & 5.61 \\
\textbf{ACE (Ours)} & \textbf{5.93} & \textbf{6.08} & & \textbf{5.57} & \textbf{5.24} & & \textbf{5.82} & \textbf{5.31} & & \textbf{6.51} & \textbf{5.89} \\
\midrule
VCD & 5.51 & 5.58 & & 5.04 & 5.26 & & 5.26 & 5.31 & & 5.88 & 5.53 \\
\textbf{ACE (Ours)} & \textbf{6.12} & \textbf{5.97} & & \textbf{5.41} & \textbf{5.49} & & \textbf{5.68} & \textbf{5.44} & & \textbf{6.19} & \textbf{5.84} \\
\midrule
OPERA & 6.18 & 5.52 & & 5.22 & 4.91 & & 5.39 & 4.82 & & 6.07 & 5.29 \\
\textbf{ACE (Ours)} & 6.13 & \textbf{5.99} & & \textbf{5.71} & \textbf{5.38} & & \textbf{5.91} & \textbf{5.92} & & \textbf{6.58} & \textbf{6.27} \\
\midrule
SID & 5.86 & 5.68 & & 5.35 & 5.09 & & 5.41 & 5.38 & & 6.17 & 5.85 \\
\rowcolor[HTML]{EFEFEF}
\textbf{ACE (Ours)} & \textbf{6.25} & \textbf{6.02} & & \textbf{5.84} & \textbf{5.51} & & \textbf{6.07} & \textbf{5.96} & & \textbf{6.75} & \textbf{6.32} \\
\bottomrule
\end{tabular}
}
\end{table*}

Evaluation Protocol. Following the methodology in \cite{huang2024opera, ICLR2025_3cc87f2b}, we randomly sample 500 images from the MSCOCO dataset. For each image, we prompt the MLLMs with: \textit{``Please describe this image in detail''} with a maximum length of 512 tokens. To ensure a fair and unbiased comparison, GPT-4V acts as a blind judge. It is presented with the original image and a pair of anonymized responses, one from a baseline method and one from ACE.

The sequence of responses is shuffled to eliminate any potential position bias. GPT-4V scores each response on a scale of 0 to 10 based on two criteria: \begin{itemize} \item \textbf{Correctness:} Evaluates the factual consistency with the image. Hallucinated attributes, such as wrong counts, colors, or spatial relations, significantly penalize this score. \item \textbf{Detailedness:} Assesses the richness and granularity of the description, excluding any illusory details. \end{itemize}

Results Analysis. As summarized in Table~\ref{tab:gpt4v_eval}, our ACE framework consistently outperforms competitive baselines across all four MLLMs. Specifically, compared to standard Sampling, ACE improves the Correctness (C) score by a significant margin (e.g., 5.39 $\rightarrow$ 6.51 on LLaVA-NeXT) while simultaneously enhancing the Detailedness (D) level. When compared against advanced methods like VCD and OPERA, ACE demonstrates a superior ability to suppress hallucinations without incurring a "verbosity penalty."Notably, while OPERA occasionally achieves high correctness (e.g., 6.18 on LLaVA-1.5), its detailedness often lags behind (e.g., 5.52 vs. 5.99 for ACE), indicating a tendency toward over-conservative, truncated descriptions.In contrast, ACE maintains high detailedness scores across all settings, validating that our adversarial equilibrium effectively resurrects structurally silenced visual facts rather than simply pruning text. Furthermore, ACE consistently surpasses its predecessor, SID, across all metrics, most notably on the deeper LLaVA-NeXT architecture (6.75 C and 6.32 D), confirming that mid-layer rectification is essential for maintaining factual grounding in large-scale models. From a human-perceptual perspective, these results underscore ACE's robustness in generating both trustworthy and informative descriptions.
%%%%%%%%%%%%%%%%%%%%%%%%%%%%%%%%%%%%%%%%%%%%%%%%%%%%%%%%%%%%%%%%%%%%%%%%%%%%%%%
%%%%%%%%%%%%%%%%%%%%%%%%%%%%%%%%%%%%%%%%%%%%%%%%%%%%%%%%%%%%%%%%%%%%%%%%%%%%%%%

\end{document}